\documentclass{article}

\usepackage[preprint]{neurips_2022}

\usepackage[utf8]{inputenc} 
\usepackage[T1]{fontenc}    
\usepackage{hyperref}       
\usepackage{url}            
\usepackage{booktabs}       
\usepackage{amsfonts}       
\usepackage{nicefrac}       
\usepackage{microtype}      
\usepackage{xcolor}         
\usepackage{amsmath}
\usepackage{amssymb}
\usepackage{mathtools}
\usepackage{amsthm}
\usepackage{multirow}
\usepackage{subcaption}
\usepackage{comment}
\usepackage{wrapfig}
\usepackage{float}
\usepackage{algorithm2e}
\RestyleAlgo{ruled}

\newcommand{\norm}[1]{\left\lVert #1 \right\rVert}

\makeatletter
\newcommand{\subalign}[1]{%
  \vcenter{%
    \Let@ \restore@math@cr \default@tag
    \baselineskip\fontdimen10 \scriptfont\tw@
    \advance\baselineskip\fontdimen12 \scriptfont\tw@
    \lineskip\thr@@\fontdimen8 \scriptfont\thr@@
    \lineskiplimit\lineskip
    \ialign{\hfil$\m@th\scriptstyle##$&$\m@th\scriptstyle{}##$\hfil\crcr
      #1\crcr
    }%
  }%
}
\makeatother

\theoremstyle{plain}
\newtheorem{theorem}{Theorem}[section]

\newtheorem{lemma}[theorem]{Lemma}
\newtheorem{corollary}[theorem]{Corollary}
\theoremstyle{definition}
\newtheorem{definition}[theorem]{Definition}

\theoremstyle{remark}

\DeclareMathOperator*{\argmax}{arg\,max}

\title{Your diffusion model secretly knows the dimension of the data manifold}

\author{
   Jan Stanczuk $ ^\ast$ \\
   University of Cambridge \\
   Cambridge CB3 0WA \\
   \texttt{js2164@cam.ac.uk}
   \And
   Georgios Batzolis \thanks{Equal contibution.} \\
   University of Cambridge\\
   Cambridge CB3 0WA \\
   \texttt{gb511@cam.ac.uk} 
   \AND
   Teo Deveney \\
   University of Bath\\
   Bath BA2 7AY \\
   \texttt{T.J.Deveney@bath.ac.uk}
   \And
   Carola-Bibiane Sch\"{o}nlieb \\
   University of Cambridge \\
   Cambridge CB3 0WA \\
   \texttt{cbs31@cam.ac.uk}
}

\makeatletter
\newcommand{\algorithmfootnote}[2][\footnotesize]{%
  \let\old@algocf@finish\@algocf@finish
  \def\@algocf@finish{\old@algocf@finish
    \leavevmode\rlap{\begin{minipage}{\linewidth}
    #1#2
    \end{minipage}}%
  }%
}
\makeatother

\begin{document}

\maketitle
\begin{abstract}
In this work, we propose a novel framework for estimating the dimension of the data manifold using a trained diffusion model. A diffusion model approximates the score function i.e. the gradient of the log density of a noise-corrupted version of the target distribution for varying levels of corruption. We prove that, if the data concentrates around a manifold embedded in the high-dimensional ambient space, then as the level of corruption decreases, the score function points towards the manifold, as this direction becomes the direction of maximal likelihood increase. Therefore, for small levels of corruption, the diffusion model provides us with access to an approximation of the normal bundle of the data manifold. This allows us to estimate the dimension of the tangent space, thus, the intrinsic dimension of the data manifold. To the best of our knowledge, our method is the first estimator of the data manifold dimension based on diffusion models and it outperforms well established statistical estimators in controlled experiments on both Euclidean and image data.

\end{abstract}

\section{Introduction}
\label{sec:introduction}

Many modern real-world datasets contain a large number of variables, often exceeding the number of observations. This poses a major challenge in modelling them, due to the \textit{curse of dimensionality}. However, despite this complexity, such data often concentrates around a lower-dimensional manifold, a concept known as the \textit{manifold hypothesis} \cite{manifold_hypothesis}. This hypothesis has guided the development of modern high-dimensional data modeling techniques like GANs \cite{gan} and VAEs \cite{vae}, as well as dimensionality reduction methods such as PCA \cite{pca} and t-SNE \cite{tsne}. These approaches require knowledge of the data's intrinsic dimension, a critical hyperparameter. 


In this work, we introduce a novel method, which estimates the dimension of the data manifold by leveraging information contained in a trained diffusion model.


Diffusion models \cite{diffusion_models, ddpm}, are a new class of deep generative models capable of capturing complex high-dimensional distributions without requiring prior knowledge of the data's intrinsic dimension. We show that, despite their lack of \textit{explicit} dependence on it, these models \textit{implicitly} estimate the data's intrinsic dimension.


As shown in \cite{song2020score, ddpm}, diffusion models perform score matching \cite{score_matching} and, therefore, contain the information about the gradient of the log-density of the data distribution. Our approach capitalizes on the insight that near the data manifold, the gradient of the log-density is orthogonal to the manifold itself. This key observation serves as a tool for deducing the manifold's dimension.

We evaluate the performance of the method on synthetic Euclidean and image datasets, where the dimension of the data manifold is known \textit{a priori}. Moreover, we apply our method to the MNIST dataset \cite{mnist} (where the intrinsic dimension is unknown), and evaluate its performance using the reconstruction error of auto-encoders trained with different latent dimensions. 


\begin{figure}
\begin{minipage}{0.45\textwidth}
        \centering
    \includegraphics[width=0.99\textwidth]{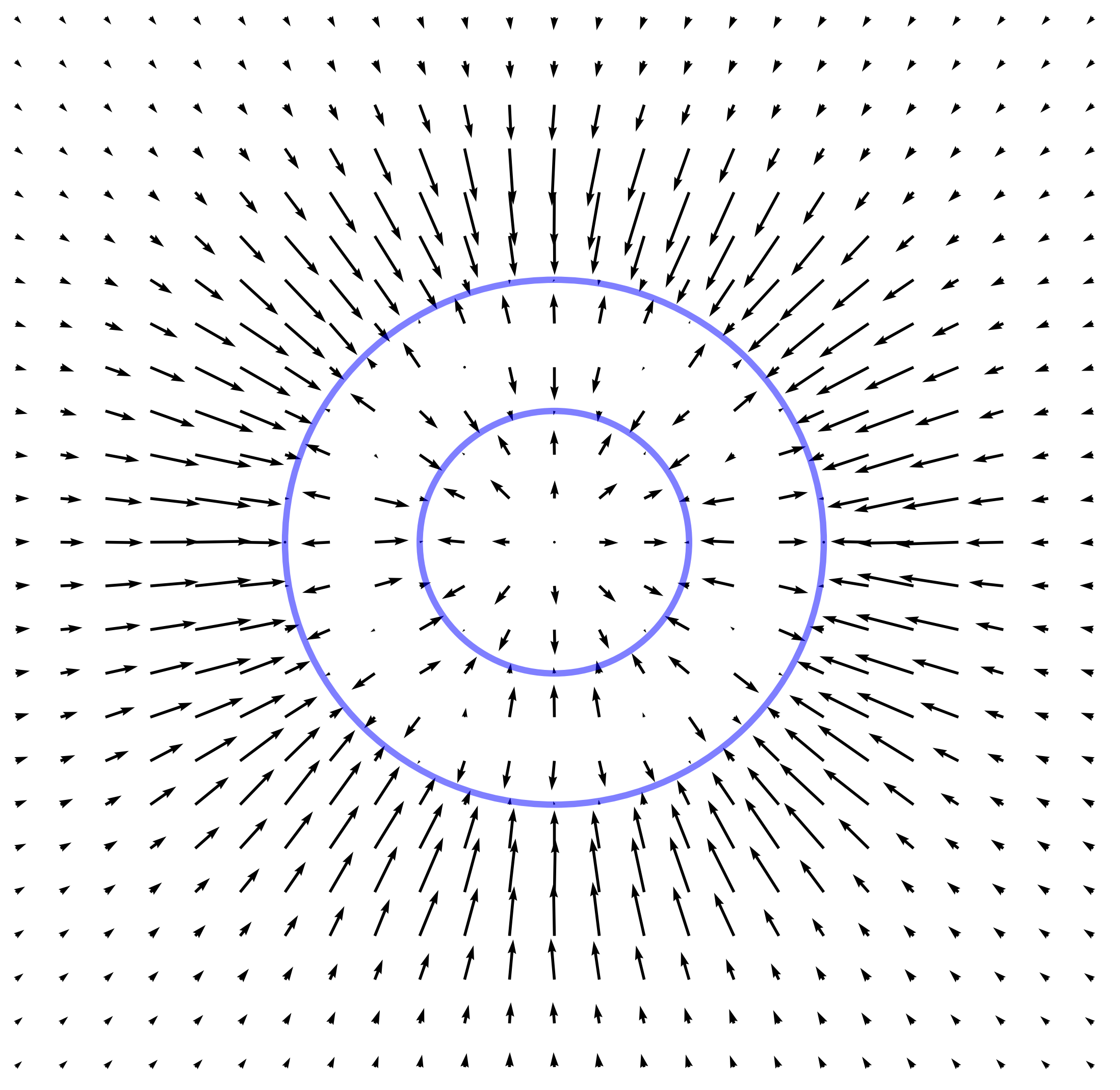}
    \caption{The data manifold (in blue) and the neural approximation of the score field $\nabla_\textbf{x} \ln p_{t_0}(\textbf{x})$ obtained from a diffusion model. Near the manifold the score field is perpendicular to the manifold surface.}
    \label{fig:score_field}
\end{minipage}
\hspace{5mm}
\begin{minipage}{0.45\textwidth}
       \centering
    \includegraphics[width=0.99\textwidth]{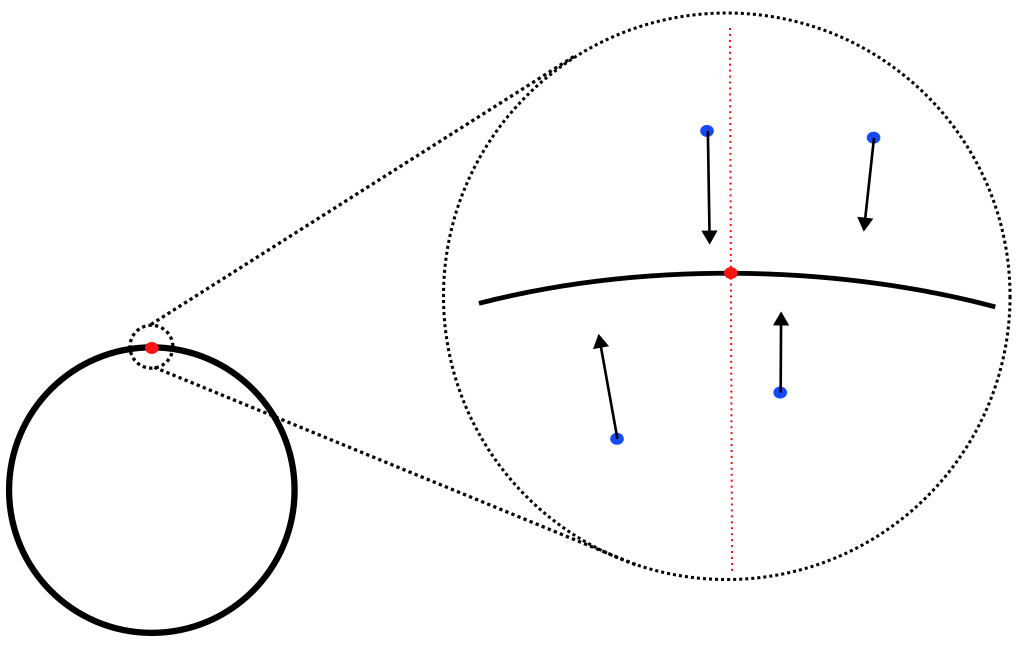}
    \caption{The red dot shows a point $\textbf{x}_0$ on the data manifold where we wish to estimate the dimension. We sample $K$ blue points $\textbf{x}_t^{(i)}$ in a close neighbourhood of the red point and evaluate the score field. The resulting vectors $s_\theta(\textbf{x}_\epsilon^{(i)}, \epsilon)$ will point in the normal direction. We put the vectors into a matrix and perform  SVD to detect the dimension of the normal space. The dimension of the manifold will be equal to the number of (almost) vanishing singular values.}
    \label{fig:zoom} 
\end{minipage}
\end{figure}
\section{Related Work}
\label{sec:related_work}
The problem of estimating the intrinsic dimensionality has been widely studied. The two main lines of research are PCA based and nearest neighbour based approaches. In an early work \cite{Karhunen-Loeve} the authors suggest an approach based on using local Karhunen–Loève expansion. In following years many PCA based approaches have been developed. Most notably, in \cite{auto_ppca} the author suggests an intrinsic dimensionality estimator based on the probabilistic PCA (PPCA) framework \cite{ppca}. In \cite{fan_local_pca} a local PCA method has been suggested. In \cite{pettis_nn_dim_estimator} authors suggested an estimator based on nearest neighbour information.  In \cite{dim_MLE} authors introduce a maximum likelihood (MLE)  procedure based on the distance to $m$ nearest neighbours. Their method has been further improved in the work of \cite{haro_mle}. The MLE method has been recently applied by \cite{pope2021intrinsic} in the estimation of the intrinsic dimensionality of modern image datasets such as MNIST \cite{mnist}, CIFAR \cite{cifar} and ImageNet \cite{imagenet}. Other works explored geometric approaches using fractal-based methods \cite{fractal_dim} or packing numbers \cite{packing_number}.

To the best of our knowledge, we are the first to present a method that approximates the intrinsic dimension of the data manifold using diffusion models.

\section{Background: Score-based diffusion models}
\label{sec:background}

In \cite{song2020score} score-based  \cite{score_matching} and diffusion-based \cite{diffusion_models, ddpm} generative models were unified into a single continuous-time score-based framework where the diffusion is represented by a stochastic differential equation, which can be reversed to generate samples. Diffusion models are trained by approximating the score function $\nabla_{\textbf{x}_t}{\ln{p_t(\textbf{x}_t)}}$ with a neural network $s_\theta(\textbf{x}_t,t)$. Additional details on training and sampling from diffusion models are described in Appendix \ref{appendix:background}.




\section{Proposed Method for Estimation of Intrinsic Dimension}
\label{sec:method}
\begin{wrapfigure}{R}{0.55\textwidth}
\begin{minipage}{0.55\textwidth}
    \begin{algorithm}[H]
    \caption{Estimate the intrinsic dimension at $\textbf{x}_0$}
    \label{alg:pseudo_code}
    \algorithmfootnote{where $(s_i)_{i=1}^d$, $(v_i)_{i=1}^d$, $(w_i)_{i=1}^d$ denote singular values, left and right singular vectors respectively.}
    \SetAlgoLined
    \DontPrintSemicolon
    
    \KwIn{$s_\theta$ - trained diffusion model (score), \\
        \hspace{10mm} $t_0$ - sampling time, \\
        \hspace{10mm} $K$ - number of score vectors.}
        \ \;
     Sample $\textbf{x}_0 \sim p_0(\textbf{x})$ from the data set \;
     $d \gets \text{dim}(\textbf{x}_0)$ \;
     $S \gets $ empty matrix \;
     \For{$i = 1, ..., K$}{
        Sample $\textbf{x}_{t_0}^{(i)} \sim \mathcal{N}(\textbf{x}_{t_0} | \textbf{x}_0, \sigma^2_{t_0}\textbf{I}) $ \;
        Append $s_{\theta}(\textbf{x}_{t_0}^{(i)}, t_0)$ as a new column to $S$ \;
    }
    $(s_i)_{i=1}^d$, $(\textbf{v}_i)_{i=1}^d$, $(\textbf{w}_i)_{i=1}^d$ $\gets \text{SVD}(S)$  \;
    $\hat{k}(\textbf{x}_0) \gets d - \argmax_{i=1,..,d-1} (s_i - s_{i+1}) $ \;
    \KwOut{ $\hat{k}(\textbf{x}_0)$ }
\end{algorithm}
\end{minipage}
\end{wrapfigure}

Consider a dataset $D=\{\textbf{x}^{(i)}\}_{i=0}^N\sim p_0(\textbf{x})$ which consists of $N$ independent $d$-dimensional vectors $\textbf{x}^{(i)} \in \mathbb{R}^d$ drawn from distribution $p_0(\textbf{x})$. The distribution $p_0(\textbf{x})$ is supported on a $k$-dimensional manifold $\mathcal{M}$, which is embedded in a space of ambient dimension $d$. Our goal is to infer the dimension $k$ of the manifold $\mathcal{M}$ from $D$.

We perturb the data according to the variance exploding SDE \(d\textbf{x}_t = g(t)d\textbf{w}_t\)  \cite{song2020score} and train a neural network $s_{\theta}(\textbf{x}_t,t)$ to approximate the score function of the noise perturbed target distribution, i.e. $\nabla_{\textbf{x}_t}{\ln{p_t(\textbf{x}_t)}}$ for a range of levels of perturbation indexed by diffusion time $t$. We train the model using the weighted denoising score matching objective with likelihood weighting, see \cite{song2021maximum}. More details about the training of the diffusion model can be found in Appendix \ref{sec:hparams}. 

Consider a datapoint $\textbf{x}_0$ on a manifold $\mathcal{M}$ and perturb it into the ambient space with the transition kernel $p_{t_0}(\textbf{x}_{t_0} | \textbf{x}_0) = \mathcal{N}(\textbf{x}_{t_0} | \textbf{x}_0, \sigma^2_{t_0}\textbf{I})$\footnote{The transition kernels of the variance exploding SDE have this structure, where $\sigma_t$ is an increasing function determined by $g(t)$ with $\sigma_t \xrightarrow[t \to 0]{} 0$.} of the forward process for a small time ${t_0}$. As shown in the next section, at $\textbf{x}_{t_0}$ the score vector $s_{\theta}(\textbf{x}_{t_0},{t_0})$ will point towards its orthogonal projection onto $\mathcal{M}$, making it almost orthogonal to $T_{\textbf{x}_0}\mathcal{M}$ (the tangent space at $\textbf{x}_0$). This score vector will be almost entirely contained in $\mathcal{N}_{\textbf{x}_0}\mathcal{M}$ (the normal space at $\textbf{x}_0$) as a consequence of this, and so the rank of a matrix containing score vectors evaluated at $K$ points diffused from $\textbf{x}_0$ should not exceed the dimension of $\mathcal{N}_{\textbf{x}_0}\mathcal{M}$. With enough samples, the rank of $S=[s_{\theta}(\textbf{x}_{t_0}^{(1)},t_0),...,s_{\theta}(\textbf{x}_{t_0}^{(K)},t_0)]$ estimates the normal space dimension, from which the manifold dimension can be estimated. In our method, we sample $K=4d$ diffused points at time $t_0=\epsilon$ and calculate the SVD of $S$, finally we estimate the intrinsic dimension $\hat{k}(\textbf{x}_0)$ as the number of vanishing singular values.

The resulting spectrum shows a significant drop exactly or very close to the dimension of the normal space. The remaining non-zero but much smaller singular values correspond to the tangential component of the score vector. This behaviour is expected as the score vector will unavoidably have a very small tangential component, for reasons explained in the following sections. The choice of the cut-off point $\hat{k}(\textbf{x}_0)$ is usually very clear visually, but can be automated by choosing the point of largest drop in the spectrum:
\begin{gather*}
  \hat{k}(\textbf{x}_0) = d - \argmax_{i=1,..,d-1} s_i - s_{i+1} 
\end{gather*}


When selecting $\textbf{x}_0$, we ideally want a point with high score approximation quality, minimal tangential component, and low manifold curvature. However, since these factors are uncontrollable, we randomly choose multiple $\textbf{x}_0^{(j)}$ values and plot a spectrum for each. For simple distributions, the score spectra look similar, with drops at accurate values. For more complex distributions, the drop location varies with $\textbf{x}_0^{(j)}$ choice. We find that the maximum estimated $\hat{k}$ gives the best estimate. Theoretical understanding of the method supports this, as discussed in later sections.


\section{Theoretical Analysis}
\label{sec:theory}

Here we provide a theoretical justification of our approach; showing that given a collection of points $\textbf{x}_i \in \mathbb{R}^d$ sufficiently close\footnote{For `sufficiently close' to be defined in Appendix \ref{appendix:proof}} to the manifold with orthogonal projection $\pi(\textbf{x}_i)\in \mathcal{M}$, in the small $t$ limit the space spanned by score vectors $\nabla_{\textbf{x}} \ln p_t(\textbf{x}_i)$ converges to the normal space at $\pi(\textbf{x}_i)$. To build intuition, consider a uniform data density on the manifold surface $\mathcal{M}$. The gradient along this density is zero, indicating that for $\textbf{x}$ close to  $\mathcal{M}$ tangential components of the score $\nabla_\textbf{x} \ln p_t(\textbf{x})$ will also be approximately zero and the score  will be mostly contained in the normal bundle $\mathcal{NM}$. If the density is non-uniform on the manifold surface however, the score will have a tangential component. Fortunately, for sufficiently small $t$ the change in log-density from moving orthogonally towards the manifold  dominates the change from moving tangentially alongside the manifold. This results in the tangential component becoming negligible, and the score still being approximately contained in $\mathcal{NM}$.

Specifically, we show in the following theorem that for any point $\textbf{x}$ sufficiently close to the data manifold and $t\to 0$, the score $\nabla_\textbf{x} \ln p_t(\textbf{x})$ is points directly at the orthogonal projection $\pi(\textbf{x})$.

\begin{theorem}
\label{thm:score_orthogonal}
Suppose that the the support of the data distribution $P_0$ is contained in a compact embedded sub-manifold $\mathcal{M} \subseteq \mathbb{R}^d$ and let $P_t$ be the distribution of samples from $P_0$ diffused for time $t$. Then, under mild assumptions, for any point $\textup{\textbf{x}} \in \mathbb{R}^d$ sufficiently close to $\mathcal{M}$ with orthogonal projection on $\mathcal{M}$, given by $\pi(\textup{\textbf{x}})$, and $\textup{\textbf{n}} = (\pi(\textup{\textbf{x}}) - \textup{\textbf{x}}) / \norm{\pi(\textup{\textbf{x}}) - \textup{\textbf{x}}}$, we have: 
\begin{gather*}
 \text{S}_{\cos} (\textup{\textbf{n}}, \nabla_\textbf{x} \ln p_t(\textup{\textbf{x}})) \xrightarrow[t \to 0]{} 1
\end{gather*}
where $\text{S}_{\cos}$ denotes the cosine similarity. In other words, for sufficiently small $t$ the score $\nabla_\textbf{x} \ln p_t(\textup{\textbf{x}})$ points directly at the projection of $\textup{\textbf{x}}$ on the manifold. 
\end{theorem}
This theorem leads to the conclusion that this score is contained within the normal space of the manifold, as we show with the following corollary:
\begin{corollary}
\label{cor:score_ratio}
The ratio of the projection of the score $\nabla_\textbf{x} \ln p_t(\textup{\textbf{x}})$ on the tangent space of the data manifold $T_{\pi(\textup{\textbf{x}})}\mathcal{M}$ to the projection on the normal space $\mathcal{N}_{\pi(\textup{\textbf{x}})}\mathcal{M}$ approaches zero as $t$ approaches zero, i.e.
\begin{gather*}
\frac{\norm{\mathbf{T}\nabla_\textbf{x} \ln p_t(\textup{\textbf{x}})}}{\norm{\mathbf{N}\nabla_x \ln p_t(\textup{\textbf{x}})}} \to 0, \text{ as } t \to 0. 
\end{gather*}
where $\mathbf{N}$ and $\mathbf{T}$ are projection matrices on $\mathcal{N}_{\pi(\textup{\textbf{x}})}\mathcal{M}$  and $T_{\pi(\textup{\textbf{x}})}\mathcal{M}$ respectively. Therefore for sufficiently small $t$ the score $\nabla_\textup{\textbf{x}} \ln p_t(\textup{\textbf{x}})$ is (effectively) contained in the normal space  $\mathcal{N}_{\pi(\textup{\textbf{x}})}\mathcal{M}$.
\end{corollary}

\textit{Proof}. The full proof of the theorem and corollary can be found in the Appendix \ref{appendix:proof}.

In practice we choose small $t>0$, and for each chosen $\textbf{x}_0\in \mathcal{M}$ we sample points $\textbf{x}^{(i)}_t$ around it from $p_t(\textbf{x}_t | \textbf{x}_0) = \mathcal{N}(\textbf{x}_t | \textbf{x}_0, \sigma^2_t\textbf{I})$. With over $99\%$ probability $\textbf{x}^{(i)}_t \in B(\textbf{x}_0, 3\sigma_t)$, and so as we decrease $t$ most of our $\textbf{x}_t$ will become very close to $\textbf{x}_0$. For $B(\textbf{x}_0, 3\sigma_t)$ sufficiently small, the effect of curvature of $\mathcal{M}$ becomes negligible inside $B(\textbf{x}_0, 3\sigma_t)$, and so the normal spaces $\mathcal{N}_{\pi(\textbf{x}_t^{(i)})}\mathcal{M}$ will all be approximately equal to $\mathcal{N}_{\textbf{x}_0}\mathcal{M}$. Under this assumption, we can outline the practical implications of these theoretical results. Assuming $t>0$ small and a trained score approximation $s_\theta(\textbf{x}, t) \approx \nabla_\textbf{x} \ln p_t(\textbf{x})$, the score matrix $S=[s_{\theta}(\textbf{x}_t^{(1)},t),...,s_{\theta}(\textbf{x}_t^{(4d)},t)]$ has the following properties:
\begin{enumerate}
    \item The columns of $S$ are approximately contained in the normal space $\mathcal{N}_{\textbf{x}_0}\mathcal{M}$
    \item The columns of $S$ approximately span the normal space $\mathcal{N}_{\textbf{x}_0}\mathcal{M}$
    \item The singular values of $S$ corresponding to singular vectors in the normal space are large relative to those corresponding to tangent singular vectors 
\end{enumerate}
The first point is a direct consequence of Corollary \ref{cor:score_ratio}. For the second point, denote $n_\textbf{x}:=\frac{\pi(\textbf{x})-\textbf{x}}{\norm{\pi(\textbf{x})-\textbf{x}}}$, and assume that $t$ is sufficiently small such that $\mathcal{N}_{\pi(\textbf{x}_t^{(i)})}\mathcal{M} \approx \mathcal{N}_{\textbf{x}_0}\mathcal{M}$. Then locally, the vectors $\left\{n_{\textbf{x}_t^{(1)}},\dots,n_{\textbf{x}_t^{(K)}} \right\}$ are independent Gaussian perturbations from a linear subspace. With probability one this set contains $N_d = \text{dim}(\mathcal{N}_{\textbf{x}_0}\mathcal{M})$ linearly independent vectors spanning $\mathcal{N}_{\textbf{x}_0}\mathcal{M}$, so by Theorem \ref{thm:score_orthogonal} the score vectors  $\left\{ \nabla_{\textbf{x}} \ln p_t\big(\textbf{x}_t^{(1)}\big),\dots, \nabla_{\textbf{x}} \ln p_t\big(\textbf{x}_t^{(K)}\big) \right\}$ must therefore also span $\mathcal{N}_{\textbf{x}_0}\mathcal{M}$. For the third point note that if the columns of $S$ span $\mathcal{N}_{\textbf{x}_0}\mathcal{M}$, then $S$ has rank $N_d$, and its SVD yields singular values $s_i>0$ for $i\leq N_d$ corresponding to singular vectors in $\mathcal{N}_{\textbf{x}_0}\mathcal{M}$, and $s_j=0$ for $j>N_d$. In practice we fix some small $t>0$ and so small components of $T_{\textbf{x}_0}\mathcal{M}$ are introduced to the score by factors such as non-uniform distribution of data samples on $\mathcal{M}$, therefore in applications the SVD of $S$ yields small singular values $s_j>0$ for $j>N_d$, however we still observe that $s_i \gg s_j$ where $i \leq N_d$.

\section{Limitations}
\label{sec:limitations}

In section \ref{sec:theory}, we established that given a perfect score approximation for sufficiently low $t$ our method produces the correct estimation of the dimension. However, in practice, our method may encounter two types of errors: \textit{approximation error} and \textit{geometric error}. The approximation error arises as a result of having an imperfect score approximation $s_\theta(\textbf{x}, t) \approx \nabla_\textbf{x} \ln p_t(\textbf{x})$. 
Geometric error arises if the selected sampling time $t$ isn't sufficiently small, potentially impacting our method's accuracy for two reasons. Firstly, it may result in an increased tangential component of the score vector. Secondly, if $\textbf{x}^{(i)}_t$ lies too distant from $\mathcal{M}$, the manifold's curvature may create a difference between normal spaces $\mathcal{N}_{\pi(\textbf{x}_t^{(i)})}\mathcal{M}$ across varying $i$.


We empirically assess our method's robustness to approximation error and find it robust to minor inaccuracies in score approximation. Additionally, we analyze our method's sensitivity to $p_0$ non-uniformity, which could induce a minor tangential score component for $t>0$. We discover that using the maximum $\hat{k}(\textbf{x}^{(j)}_0)$ allows our method to accommodate varying levels of non-uniformity over the manifold surface, displaying superior robustness compared to other non-linear estimators (MLE, Local PCA) without the need for reducing $t$. Details are in Appendix \ref{appendix:robust}


We note that Theorem \ref{thm:score_orthogonal} assumes the data distribution's support is exclusively within a manifold. Therefore, we empirically investigate the method's applicability when data is concentrated around, but not entirely within, a manifold. We discover that for a $k$-sphere, our method remains reliable as long as the data is closely concentrated around the manifold. Details are available in Appendix \ref{appendix:robust}.


\section{Experiments}
\label{sec:experiments}

We examine the effectiveness of our method on a multitude of manifold datasets, each embedded in a high-dimensional ambient space. The datasets fall into two categories: \textit{Euclidean datasets} consisting of points from manifolds embedded  in a high-dimensional Euclidean space and \textit{image datasets} consisting of synthetic manifolds of images.
In each case we know the intrinsic dimension of the manifold \textit{a priori}. Additionally, we apply our method to the MNIST dataset, where the true intrinsic dimension is unknown. We assess its performance via comparison with the reconstruction error of auto-encoders with various latent dimensions.  For each dataset we train a diffusion model, and then apply our method to estimate the intrinsic dimension of the data manifold. Details on hyperparameters and architectures used in our experiments can be found in Appendix \ref{sec:hparams}. 

We compare our method against established approaches to intrinsic dimensionality estimation: the nearest neighbour based maximum likelihood estimator (MLE) \cite{dim_MLE}, \cite{haro_mle}, Local PCA \cite{fan_local_pca} and Probabilistic PCA (PPCA) \cite{auto_ppca} \cite{ppca}. The details about the implementation of the benchmarks are in the Appendix \ref{sec:benchmark}.

Our method consistently yields the best estimate or close to the best estimate among considered approaches. In the following subsections we present a detailed discussion of each experiment. The results are summarised in Table \ref{tbl:results}.

\subsection{Experiments on Euclidean datasets}
\textbf{Embedded $k$-spheres:} We examined our method on $k$ dimensional spheres embedded in a $d=100$ dimensional ambient space via a random isometric embedding\footnote{To obtain a random isometric embedding we first generate the $k$ sphere in a $k+1$ dimensional small ambient space. Then we sample a random $d\times (k+1)$ Gaussian matrix $A$. We perform a QR decomposition $A = QR$. Finally we use the $d\times (k+1)$ isometry matrix $Q$ to embed the small $k+1$ dimensional space containing our manifold in the large $d$ dimensional ambient space.}. We consider two cases $k_1=10$ and $k_2=50$. The spectra of resulting score matrices are presented in Figure \ref{fig:ksphere}. Our method gives estimates of $\hat{k}_1 = 11$ and $\hat{k}_2 = 51$, which are very close to the true intrinsic dimensionality of the manifolds.

\begin{figure}
    \centering
    \includegraphics[width=.45\textwidth]{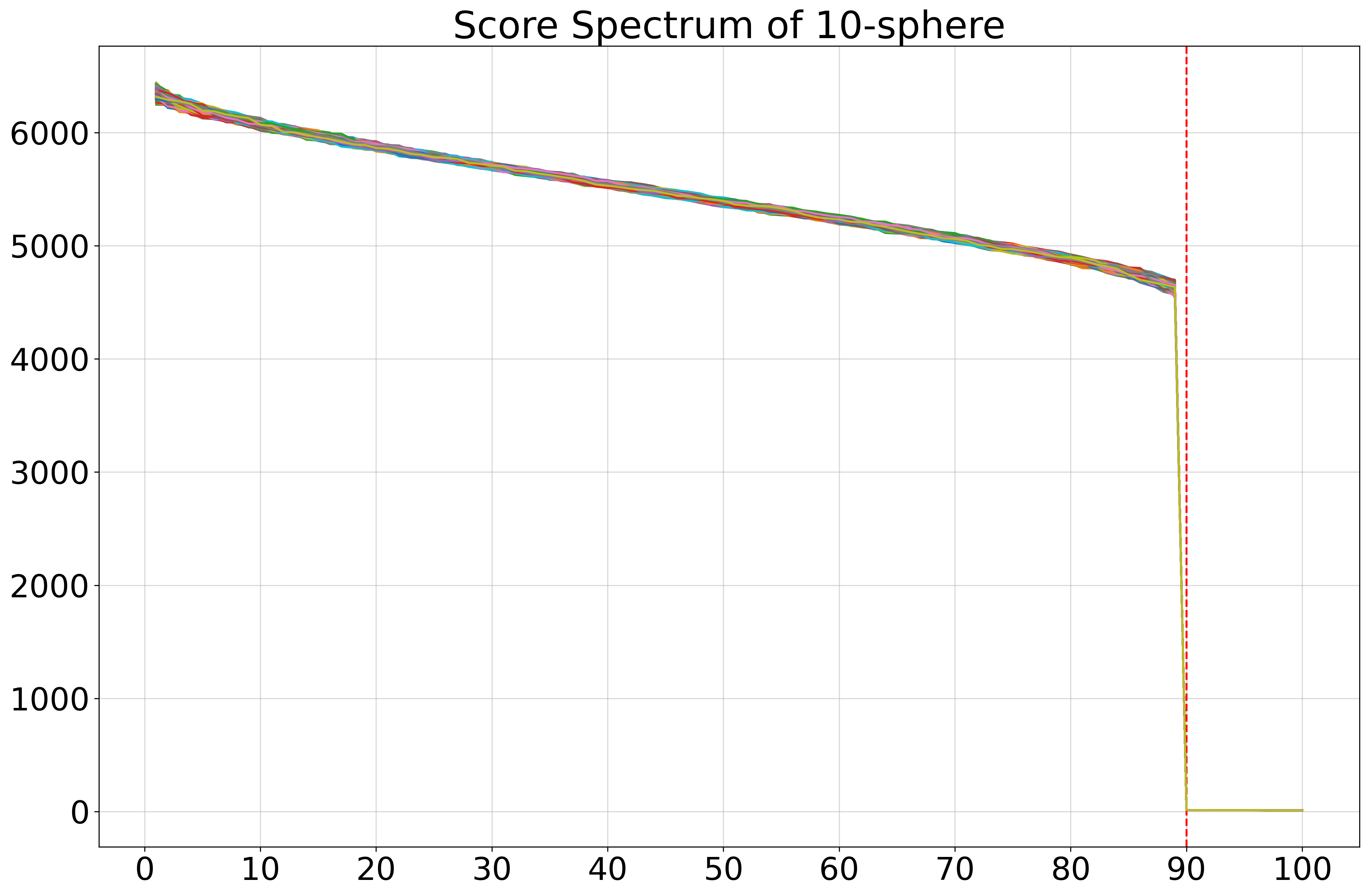}
    \hspace{5mm}
    \includegraphics[width=.45\textwidth]{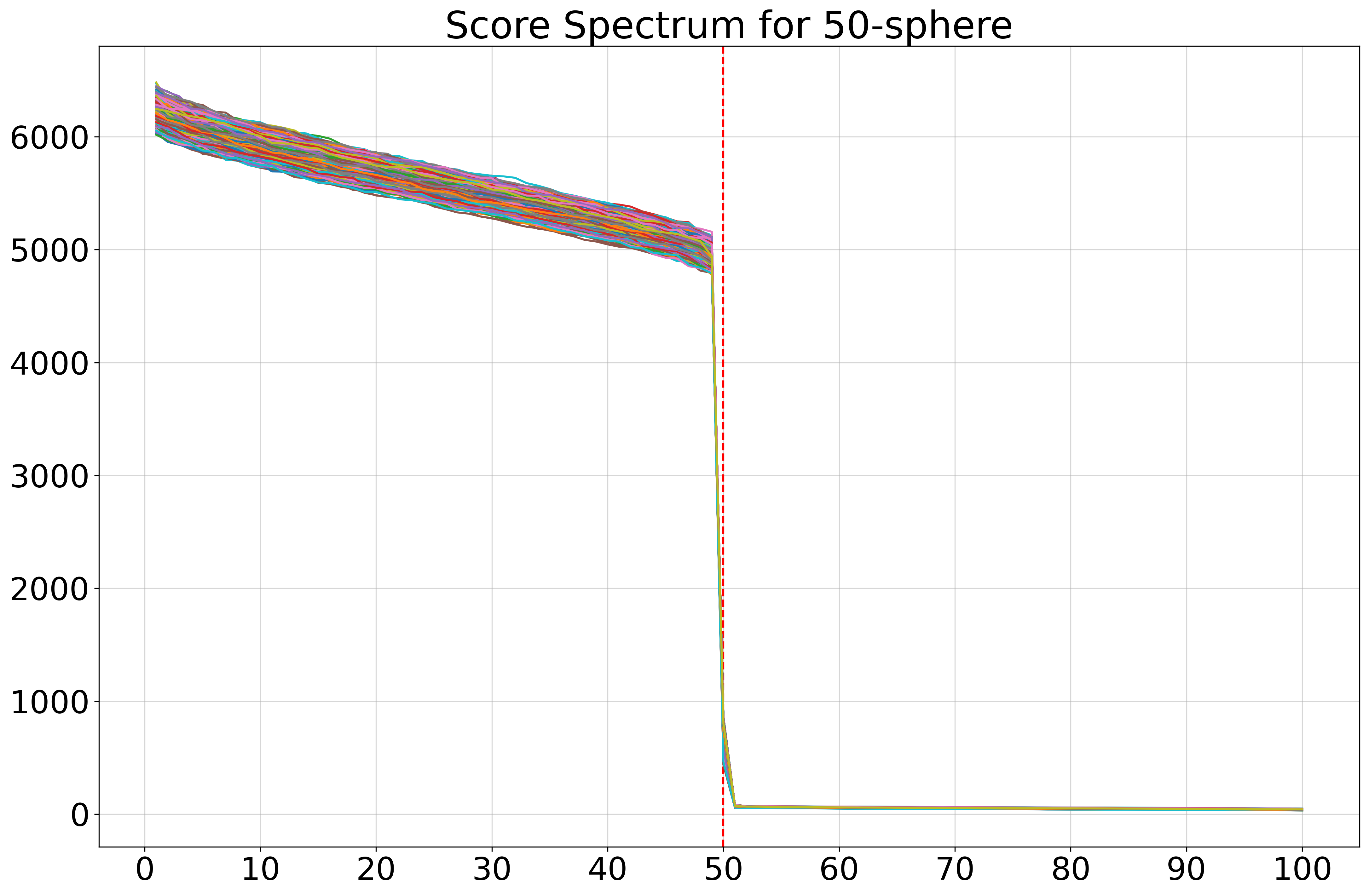}
    \caption{Singular values for the scores of $k$-sphere for $k=10, 50$. 
    In both cases around $k$ singular values almost vanish, clearly indicating the dimensionality of the manifold. Each line shows a score spectrum at different $\textbf{x}_0^{(j)}$.}
    \label{fig:ksphere}
\end{figure}




\textbf{Spaghetti line:} The intrinsic dimensionality of $k$-spheres could be well approximated by a linear dimensionality detection methods such as \cite{auto_ppca}. This is because these manifolds are contained in a low dimensional \textit{linear} subspace. In order to showcase the advantage of the \textit{non-linear} nature of our method we consider a \textit{spaghetti line} manifold. That is a curve $t \mapsto (\sin(t), \sin(2t), ..., \sin(100t))$ in a 100 dimensional ambient space, which is not contained in any low dimensional linear subspace (cf. Figure \ref{fig:line_samples} in Appendix \ref{appendix:additional_euclidean}). As expected, the linear method \cite{auto_ppca} greatly overstates the intrinsic dimension with a result of $\hat{k}_{\text{PPCA}}=98$. Yet, our approach utilizes the non-linear knowledge from the diffusion model to accurately predict an intrinsic dimensionality of one.  The score spectrum is presented in Figure \ref{fig:line} in Appendix \ref{appendix:additional_euclidean}.

\textbf{Union of $k$-spheres:} Due to the local nature of our method, we are able to generate an estimate $\hat{k}(\textbf{x}_0)$ of the intrinsic dimension around a given point $\textbf{x}_0$. This allows us to apply our approach to a union of manifolds and identify the dimension of each component. We illustrate this feature with the following experiment. We embed two spheres of different radii and dimensions in a 100 dimensional ambient space. First sphere has dimension $k_1 = 10$ and radius $r_1 = 1$ and the second sphere has $k_2=30$ and radius $r_2 = 0.25$\footnote{One can intuitively think of this manifold as a high-dimensional analog of a planet with a ring around it.}. We apply our method to this data using multiple $\textbf{x}_0^{(i)}$ randomly sampled from the dataset. We observe that our method produces a spectrum with two visible, separated drops. This indicates that the data comes from the union of manifolds of different dimensions. The resulting estimates are $\hat{k}_1 = 10$ and $\hat{k}_2=31$ depending on the chosen  $\textbf{x}_0^{(j)}$. The score spectra and the histogram of estimated dimensions are presented in Figures \ref{fig:unions_spectrum} and \ref{fig:union_dims} in Appendix \ref{appendix:additional_euclidean}.


\subsection{Experiments on image datasets}
\textbf{Synthetic image manifolds}: In this experiment, we investigated our method's ability to infer the dimension of synthetic image manifolds with known dimension. We crafted two synthetic image manifolds with controllable intrinsic dimension $k$: the "$k$ squares images manifold" and the "$k$ Gaussian blobs images manifold". The construction of these manifolds is detailed in Appendix \ref{Appendix: design of synthetic image manifolds}. We evaluated our method on $k=$10, 20, and 100 dimensional manifolds for both types, with the score spectra and histograms of estimated dimensions for numerous data points displayed in Figures \ref{fig:squares_spectrum} and \ref{fig:gaussians_spectrum} in Appendix \ref{appendix:additional_image}.

On the squares image manifold, only our method and PPCA consistently yielded accurate dimension estimates. PPCA's success on this dataset was anticipated since the manifold resides within a $k$-dimensional linear subspace.

On the Gaussian blobs image manifold, our method stood out as the sole technique to consistently deliver accurate dimension estimates.  The accuracy of our method was not compromised by the manifold's increased complexity, unlike other methods. However, the estimation for the 100-dimensional manifold introduced some uncertainty, as indicated by a more leveled histogram and a less abrupt spectrum collapse (c.f. Figure \ref{fig:gaussians_spectrum}). This is attributed to the manifold's increased complexity and the inherent challenges in optimization, resulting in greater geometrical and statistical errors.



\begin{wrapfigure}{l}{0.5\textwidth}
    \includegraphics[width=\linewidth]{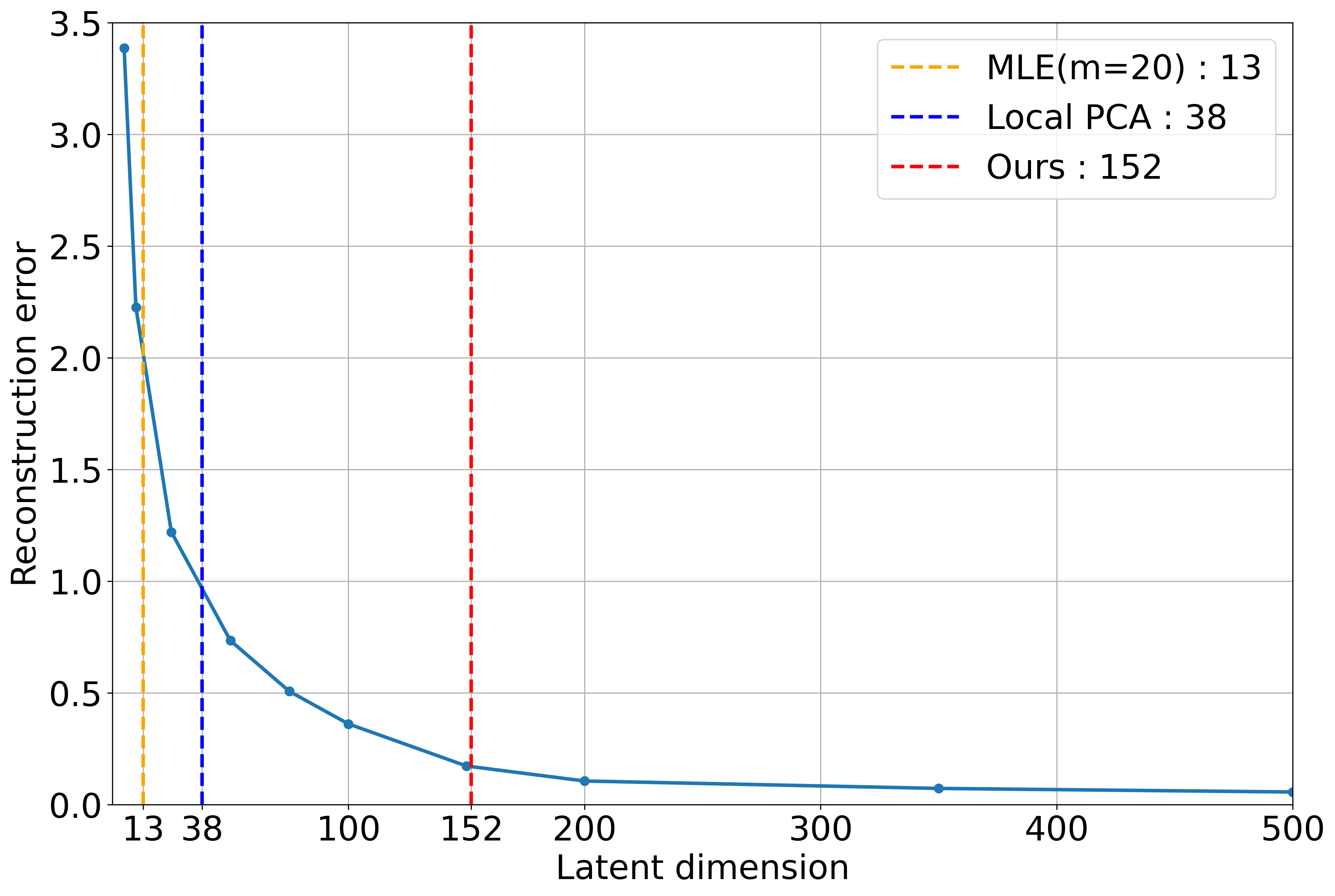}
    \caption{Auto-encoder reconstruction error on MNIST for different latent space dimensions. Vertical lines mark different estimations of intrinsic dimension.}
    \label{fig:mnist_autoencoder}
\end{wrapfigure}

\textbf{MNIST:} In our study, we additionally applied the proposed technique to estimate the intrinsic dimension of the well-known MNIST dataset - an image dataset with an as-of-yet-undetermined intrinsic dimension. Our findings suggest that there exists a variation in the intrinsic dimensions across different digits. For instance, the digit `1' yielded an estimated dimension of 66, whereas the digit `9' exhibited a significantly higher estimated dimension of 152. This discrepancy can be attributed to the increased geometric complexity inherent to the digit `9'. Figure \ref{fig:score_spectra_mnist} elucidates these observations by displaying the score spectra which yielded the maximum estimated dimensions for each digit. We present the estimated dimension for each digit in Table \ref{tab:estimated_mnist_dimensions} and the complete set of spectra for each digit in the Appendix \ref{sec:Additional Experimental Results for MNIST}. 

We validate our estimates by comparing them with the reconstruction error of auto-encoders trained with different latent dimensions. As shown in Figure \ref{fig:mnist_autoencoder} the intrinsic dimension obtained with our method coincides with the point of diminishing returns in terms of reduction of the reconstruction loss. Contrastingly, estimates produced by MLE and Local PCA were significantly lower, corresponding to regions of the curve where the reconstruction loss was still steeply decreasing. This suggests these methods underestimate the true manifold dimension. These findings calls for a careful interpretation of the intrinsic dimension estimates of popular machine learning datasets provided by \cite{pope2021intrinsic}, as they rely on the MLE method, which we have found to consistently underestimate manifold dimensions. On the other hand, PPCA notably overestimated the dimension, with $\hat{k}_{\text{PPCA}} = 706$.

\begin{figure}
    \centering
    \includegraphics[width=0.99\textwidth]{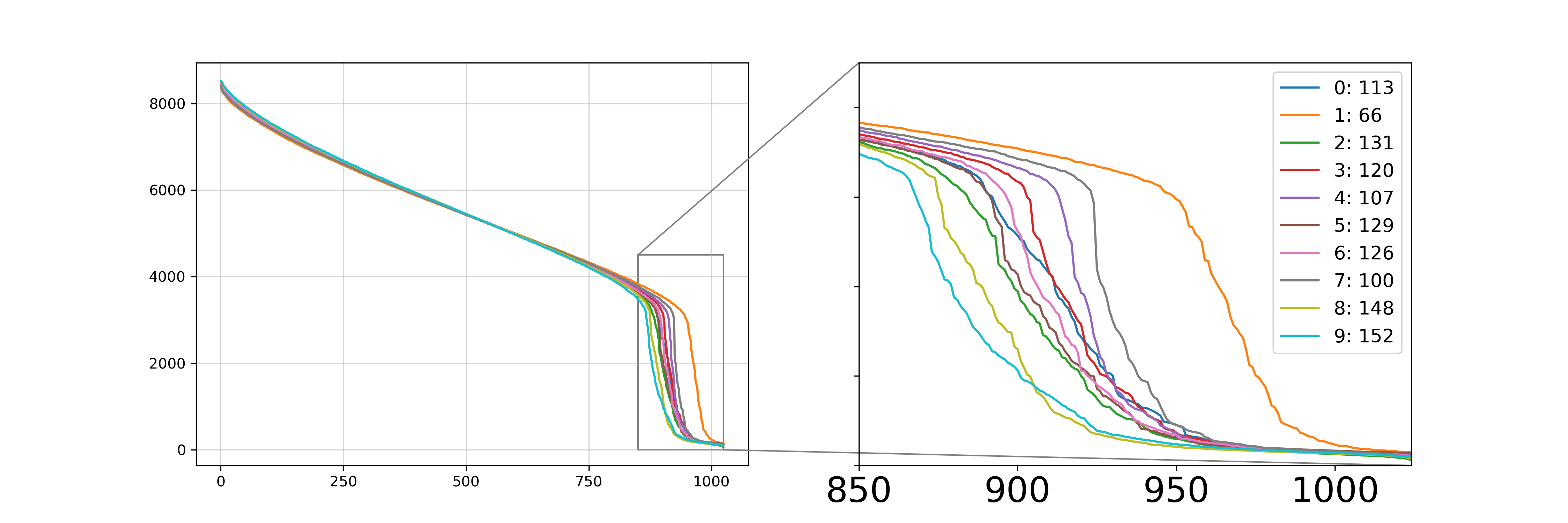}
    \caption{MNIST score spectra that yielded the highest estimated dimension for each digit}
    \label{fig:score_spectra_mnist}
\end{figure}

\begin{table*}[h]
\begin{center}
\small
\begin{tabular}{c|c|c|c|c|c|c|c|c|c}
\toprule
0 & 1 & 2 & 3 & 4 & 5 & 6 & 7 & 8 & 9 \\
\midrule
113 & 66 & 131 & 120 & 107 & 129 & 126 & 100 & 148 & 152 \\

\bottomrule
\end{tabular}
\end{center}
\caption{Estimated intrinsic dimension for each MNIST digit}
\label{tab:estimated_mnist_dimensions}
\end{table*}

%
%
%
%
%

\begin{table}[h]
    \begin{center}
    \resizebox{\columnwidth}{!}{%
    \begin{tabular}{lcccccc}
    \toprule
    & Ground Truth & Ours & MLE (m=5) & MLE (m=20) & Local PCA & PPCA \\
    \midrule
    Euclidean Data Manifolds & & & & & & \\
    \quad 10-sphere 
    &10 &11 &9.61 &9.46 &11 &11 \\
    \quad 50-sphere
    &50 &51 &35.52 &34.04 &51 &51 \\
    \quad Spaghetti line
    &1 &1 &1.01 &1.00 &32 &98 \\
    \midrule
    Image Manifolds & & & & & & \\
    \quad Squares
    & & & & & & \\
    \qquad $k=10$ 
    &10 &11 &8.48 &8.17 &10 &10 \\
    \qquad $k=20$
    &20 &22 &14.96 &14.36 &20 &20 \\
    \qquad $k=100$
    &100 &100 &37.69 &34.42 &78 &99 \\
    \quad Gaussian blobs
    & & & & & & \\
    \qquad $k=10$
    &10 &12 &8.88 &8.67 &10 &136 \\
    \qquad $k=20$
    &20 &21 &16.34 &15.75 &20 &264 \\
    \qquad $k=100$
    &100 &98 &39.66 &35.31 &18 &985 \\
    \midrule
    MNIST
    &N/A &152 &14.12 &13.27 &38 &706 \\
    \bottomrule
    \end{tabular}
    }
    \end{center}
    \caption{Comparison of dimensionality detection methods on various data manifolds.}
    \label{tbl:results}
\end{table}

\section{Conclusions and further directions}
\label{sec:conclusions}
In this work, we proved theoretically and confirmed experimentally that diffusion models can infer the intrinsic dimension from the data. We introduced an approach that estimates the intrinsic dimension of the data manifold from a pre-trained diffusion model. This approach capitalizes on the observation that, the diffusion model evaluated at sufficiently small diffusion time approximates the normal bundle of the data manifold.  Our work offers a twofold contribution: it highlights that diffusion model detects the lower dimensional structure of data and provides a rigorous method for intrinsic dimension estimation. 

Our method has been rigorously tested on Euclidean and image data and has consistently provided accurate estimates of the intrinsic dimension, outperforming established statistical estimators, especially in higher dimensional manifolds. Furthermore, our research introduces new estimates for the MNIST's dimensionality, demonstrating strong alignment with the predictions of an auto-encoder trained across a range of latent dimensions.

We are the first, to our knowledge, to present an intrinsic dimension estimation method based on diffusion models. The superior performance of our approach on higher dimensional manifolds compared to statistical estimators is attributed to the enhanced statistical efficiency from the inductive biases in the neural network architectures approximating the score function.  Our work opens new paths for understanding and estimating intrinsic data dimension, with potential implications across the field of machine learning. Future research should explore this method's applicability to other data types and its potential across various domains.

\bibliography{bibliography}

\begin{thebibliography}{31}
\providecommand{\natexlab}[1]{#1}
\providecommand{\url}[1]{\texttt{#1}}
\expandafter\ifx\csname urlstyle\endcsname\relax
  \providecommand{\doi}[1]{doi: #1}\else
  \providecommand{\doi}{doi: \begingroup \urlstyle{rm}\Url}\fi

\bibitem[Anderson(1982)]{anderson1982reverse_time_sde}
Brian~D.O. Anderson.
\newblock Reverse-time diffusion equation models.
\newblock \emph{Stochastic Processes and their Applications}, 12\penalty0
  (3):\penalty0 313--326, 1982.
\newblock ISSN 0304-4149.
\newblock \doi{https://doi.org/10.1016/0304-4149(82)90051-5}.
\newblock URL
  \url{https://www.sciencedirect.com/science/article/pii/0304414982900515}.

\bibitem[Bishop and Tipping(2001)]{ppca}
C~M Bishop and M~E Tipping.
\newblock Probabilistic principal component analysis.
\newblock 2001.
\newblock PPCA.

\bibitem[Camastra and Vinciarelli(2002)]{fractal_dim}
F.~Camastra and A.~Vinciarelli.
\newblock Estimating the intrinsic dimension of data with a fractal-based
  method.
\newblock \emph{IEEE Transactions on Pattern Analysis and Machine
  Intelligence}, 24\penalty0 (10):\penalty0 1404--1407, 2002.
\newblock \doi{10.1109/TPAMI.2002.1039212}.

\bibitem[Deng et~al.(2009)Deng, Dong, Socher, Li, Li, and Fei-Fei]{imagenet}
Jia Deng, Wei Dong, Richard Socher, Li-Jia Li, Kai Li, and Li~Fei-Fei.
\newblock Imagenet: A large-scale hierarchical image database.
\newblock In \emph{2009 IEEE Conference on Computer Vision and Pattern
  Recognition}, pages 248--255, 2009.
\newblock \doi{10.1109/CVPR.2009.5206848}.

\bibitem[Fan et~al.(2010)Fan, Gu, Qiao, and Zhang]{fan_local_pca}
Mingyu Fan, Nannan Gu, Hong Qiao, and Bo~Zhang.
\newblock Intrinsic dimension estimation of data by principal component
  analysis, 2010.
\newblock URL \url{https://arxiv.org/abs/1002.2050}.

\bibitem[Fefferman et~al.(2013)Fefferman, Mitter, and
  Narayanan]{manifold_hypothesis}
Charles Fefferman, Sanjoy Mitter, and Hariharan Narayanan.
\newblock Testing the manifold hypothesis, 2013.
\newblock URL \url{https://arxiv.org/abs/1310.0425}.

\bibitem[Fukunaga and Olsen(1971)]{Karhunen-Loeve}
K.~Fukunaga and D.R. Olsen.
\newblock An algorithm for finding intrinsic dimensionality of data.
\newblock \emph{IEEE Transactions on Computers}, C-20\penalty0 (2):\penalty0
  176--183, 1971.
\newblock \doi{10.1109/T-C.1971.223208}.

\bibitem[Goodfellow et~al.(2014)Goodfellow, Pouget-Abadie, Mirza, Xu,
  Warde-Farley, Ozair, Courville, and Bengio]{gan}
Ian~J. Goodfellow, Jean Pouget-Abadie, Mehdi Mirza, Bing Xu, David
  Warde-Farley, Sherjil Ozair, Aaron Courville, and Yoshua Bengio.
\newblock Generative adversarial networks, 2014.
\newblock URL \url{https://arxiv.org/abs/1406.2661}.

\bibitem[Haro et~al.(2008)Haro, Randall, and Sapiro]{haro_mle}
Gloria Haro, Gregory Randall, and Guillermo Sapiro.
\newblock Translated poisson mixture model for stratification learning.
\newblock \emph{Int. J. Comput. Vis.}, 80\penalty0 (3):\penalty0 358--374,
  2008.

\bibitem[Hinton and Roweis(2002)]{tsne}
Geoffrey~E Hinton and Sam Roweis.
\newblock Stochastic neighbor embedding.
\newblock In S.~Becker, S.~Thrun, and K.~Obermayer, editors, \emph{Advances in
  Neural Information Processing Systems}, volume~15. MIT Press, 2002.
\newblock URL
  \url{https://proceedings.neurips.cc/paper/2002/file/6150ccc6069bea6b5716254057a194ef-Paper.pdf}.

\bibitem[Ho et~al.(2020{\natexlab{a}})Ho, Jain, and Abbeel]{ddpm}
Jonathan Ho, Ajay Jain, and Pieter Abbeel.
\newblock Denoising diffusion probabilistic models, 2020{\natexlab{a}}.
\newblock URL \url{https://arxiv.org/abs/2006.11239}.

\bibitem[Ho et~al.(2020{\natexlab{b}})Ho, Jain, and Abbeel]{ho2020denoising}
Jonathan Ho, Ajay Jain, and Pieter Abbeel.
\newblock Denoising diffusion probabilistic models.
\newblock \emph{Advances in Neural Information Processing Systems},
  33:\penalty0 6840--6851, 2020{\natexlab{b}}.

\bibitem[Hyv{{\"a}}rinen(2005)]{score_matching}
Aapo Hyv{{\"a}}rinen.
\newblock Estimation of non-normalized statistical models by score matching.
\newblock \emph{Journal of Machine Learning Research}, 6\penalty0
  (24):\penalty0 695--709, 2005.
\newblock URL \url{http://jmlr.org/papers/v6/hyvarinen05a.html}.

\bibitem[Johnsson(2016)]{R_intrinsic_dim}
Kerstin Johnsson.
\newblock intrinsicdimension: Intrinsic dimension estimation, 2016.
\newblock URL
  \url{https://cran.r-project.org/web/packages/intrinsicDimension/}.

\bibitem[K\'{e}gl(2002)]{packing_number}
Bal\'{a}zs K\'{e}gl.
\newblock Intrinsic dimension estimation using packing numbers.
\newblock In S.~Becker, S.~Thrun, and K.~Obermayer, editors, \emph{Advances in
  Neural Information Processing Systems}, volume~15. MIT Press, 2002.
\newblock URL
  \url{https://proceedings.neurips.cc/paper/2002/file/1177967c7957072da3dc1db4ceb30e7a-Paper.pdf}.

\bibitem[Kingma and Welling(2013)]{vae}
Diederik~P Kingma and Max Welling.
\newblock Auto-encoding variational bayes, 2013.
\newblock URL \url{https://arxiv.org/abs/1312.6114}.

\bibitem[Krizhevsky(2012)]{cifar}
Alex Krizhevsky.
\newblock Learning multiple layers of features from tiny images.
\newblock \emph{University of Toronto}, 05 2012.

\bibitem[LeCun and Cortes(2010)]{mnist}
Yann LeCun and Corinna Cortes.
\newblock {MNIST} handwritten digit database.
\newblock 2010.
\newblock URL \url{http://yann.lecun.com/exdb/mnist/}.

\bibitem[Lee(2019)]{lee2019_riemman}
J.M. Lee.
\newblock \emph{Introduction to Riemannian Manifolds}.
\newblock Graduate Texts in Mathematics. Springer International Publishing,
  2019.
\newblock ISBN 9783319917542.
\newblock URL \url{https://books.google.co.uk/books?id=UIPltQEACAAJ}.

\bibitem[Levina and Bickel(2004)]{dim_MLE}
Elizaveta Levina and Peter Bickel.
\newblock Maximum likelihood estimation of intrinsic dimension.
\newblock In L.~Saul, Y.~Weiss, and L.~Bottou, editors, \emph{Advances in
  Neural Information Processing Systems}, volume~17. MIT Press, 2004.
\newblock URL
  \url{https://proceedings.neurips.cc/paper/2004/file/74934548253bcab8490ebd74afed7031-Paper.pdf}.

\bibitem[Minka(2000)]{auto_ppca}
Thomas Minka.
\newblock Automatic choice of dimensionality for pca.
\newblock In T.~Leen, T.~Dietterich, and V.~Tresp, editors, \emph{Advances in
  Neural Information Processing Systems}, volume~13. MIT Press, 2000.
\newblock URL
  \url{https://proceedings.neurips.cc/paper/2000/file/7503cfacd12053d309b6bed5c89de212-Paper.pdf}.

\bibitem[Nicolaescu(2011)]{nicolaescu2011morse_theory}
L.~Nicolaescu.
\newblock \emph{An Invitation to Morse Theory}.
\newblock Universitext. Springer New York, 2011.
\newblock ISBN 9781461411055.
\newblock URL \url{https://books.google.co.uk/books?id=nCgvt2MY4QAC}.

\bibitem[Palais and Terng(1988)]{palais1988critical}
R.S. Palais and C.~Terng.
\newblock \emph{Critical Point Theory and Submanifold Geometry}.
\newblock Critical Point Theory and Submanifold Geometry. Richard S. Palais,
  1988.
\newblock ISBN 9783540503996.
\newblock URL \url{https://books.google.co.uk/books?id=ViSHzQEACAAJ}.

\bibitem[Pearson(1901)]{pca}
Karl Pearson.
\newblock {LIII. On lines and planes of closest fit to systems of points in
  space}, November 1901.
\newblock URL \url{https://doi.org/10.1080/14786440109462720}.

\bibitem[Pedregosa et~al.(2011)Pedregosa, Varoquaux, Gramfort, Michel, Thirion,
  Grisel, Blondel, Prettenhofer, Weiss, Dubourg, et~al.]{sklearn}
Fabian Pedregosa, Ga{\"e}l Varoquaux, Alexandre Gramfort, Vincent Michel,
  Bertrand Thirion, Olivier Grisel, Mathieu Blondel, Peter Prettenhofer, Ron
  Weiss, Vincent Dubourg, et~al.
\newblock Scikit-learn: Machine learning in python.
\newblock \emph{Journal of machine learning research}, 12\penalty0
  (Oct):\penalty0 2825--2830, 2011.

\bibitem[Pettis et~al.(1979)Pettis, Bailey, Jain, and
  Dubes]{pettis_nn_dim_estimator}
Karl~W. Pettis, Thomas~A. Bailey, Anil~K. Jain, and Richard~C. Dubes.
\newblock An intrinsic dimensionality estimator from near-neighbor information.
\newblock \emph{IEEE Transactions on Pattern Analysis and Machine
  Intelligence}, PAMI-1\penalty0 (1):\penalty0 25--37, 1979.
\newblock \doi{10.1109/TPAMI.1979.4766873}.

\bibitem[Pope et~al.(2021)Pope, Zhu, Abdelkader, Goldblum, and
  Goldstein]{pope2021intrinsic}
Phillip Pope, Chen Zhu, Ahmed Abdelkader, Micah Goldblum, and Tom Goldstein.
\newblock The intrinsic dimension of images and its impact on learning.
\newblock \emph{arXiv preprint arXiv:2104.08894}, 2021.

\bibitem[Sohl-Dickstein et~al.(2015)Sohl-Dickstein, Weiss, Maheswaranathan, and
  Ganguli]{diffusion_models}
Jascha Sohl-Dickstein, Eric~A. Weiss, Niru Maheswaranathan, and Surya Ganguli.
\newblock Deep unsupervised learning using nonequilibrium thermodynamics, 2015.
\newblock URL \url{https://arxiv.org/abs/1503.03585}.

\bibitem[Song et~al.(2020)Song, Sohl-Dickstein, Kingma, Kumar, Ermon, and
  Poole]{song2020score}
Yang Song, Jascha Sohl-Dickstein, Diederik~P Kingma, Abhishek Kumar, Stefano
  Ermon, and Ben Poole.
\newblock Score-based generative modeling through stochastic differential
  equations.
\newblock \emph{arXiv preprint arXiv:2011.13456}, 2020.

\bibitem[Song et~al.(2021)Song, Durkan, Murray, and Ermon]{song2021maximum}
Yang Song, Conor Durkan, Iain Murray, and Stefano Ermon.
\newblock Maximum likelihood training of score-based diffusion models.
\newblock \emph{Advances in Neural Information Processing Systems},
  34:\penalty0 1415--1428, 2021.

\bibitem[Vincent(2011)]{vincent2011connection}
Pascal Vincent.
\newblock A connection between score matching and denoising autoencoders.
\newblock \emph{Neural Computation}, 23\penalty0 (7):\penalty0 1661--1674,
  2011.
\newblock \doi{10.1162/NECO_a_00142}.

\end{thebibliography}
\bibliographystyle{plainnat}

\newpage
\appendix
\onecolumn
\section{Extended background on diffusion models}
\label{appendix:background}
\textbf{Setup:}  In \cite{song2020score} score-based  \cite{score_matching} and diffusion-based \cite{diffusion_models, ddpm} generative models have been unified into a single continuous-time score-based framework where the diffusion is driven by a stochastic differential equation.  This framework relies on Anderson's Theorem \cite{anderson1982reverse_time_sde}, which states that under certain Lipschitz conditions on the drift coefficient $f : \mathbb{R}^{n_x} \times \mathbb{R} \xrightarrow{} \mathbb{R}^{n_x}$ and on the diffusion coefficient $G : \mathbb{R}^{n_x} \times \mathbb{R}\xrightarrow{} \mathbb{R}^{n_x} \times \mathbb{R}^{n_x}$ and an integrability condition on the target distribution $p_0(\textbf{x}_0)$ a forward diffusion process governed by the following SDE:
\begin{gather}
\label{eq:forward_sde}
 d\textbf{x}_t = f(\textbf{x}_t,t)dt+G(\textbf{x}_t,t)d\textbf{w}_t  
\end{gather} 
has a reverse diffusion process governed by the following SDE:
\label{eq:reverse_sde}
\begin{gather}
    d\textbf{x}_t=[f(\textbf{x}_t,t)-G(\textbf{x}_t,t)G(\textbf{x}_t,t)^T\nabla_{\textbf{x}_t}{\ln{p_t(\textbf{x}_t)}}]dt + G(\textbf{x}_t,t)d\Bar{\textbf{w}_t},
\end{gather}

\noindent where $\Bar{\textbf{w}_t}$ is a standard Wiener process in reverse time. 

The forward diffusion process transforms the \textit{target distribution} $p_0(\textbf{x}_0)$ to a \textit{diffused distribution} $p_T(\textbf{x}_T)$ after diffusion time $T$. By appropriately selecting the drift and the diffusion coefficients of the forward SDE, we can make sure that after sufficiently long time $T$, the diffused distribution $p_T(\textbf{x}_T)$ approximates a simple distribution, such as $\mathcal{N}(\textbf{0},\textbf{I})$. We refer to this simple distribution as the \textit{prior distribution}, denoted by $\pi$. The reverse diffusion process transforms the diffused distribution $p_T(\textbf{x}_T)$ to the data distribution $p_0(\textbf{x}_0)$ and the prior distribution $\pi$ to a distribution $p^{SDE}$. The distribution $p^{SDE}$ is close to $p_0(\textbf{x}_0)$ if the diffused distribution $p_T(\textbf{x}_T)$ is close to the prior distribution $\pi$. We get samples from $p^{SDE}$ by sampling from $\pi$ and simulating the reverse SDE from time $T$ to time $0$.

\textbf{Sampling:} To get samples by simulating the reverse SDE, we need access to the time-dependent \textit{score function} $\nabla_{\textbf{x}_t}{\ln{p_t(\textbf{x}_t)}}$. In practice, we approximate the time-dependent score function with a neural network $s_{\theta}(\textbf{x}_t,t) \approx \nabla_{\textbf{x}_t}{\ln{p_t(\textbf{x}_t)}}$ and simulate the reverse SDE presented in equation \ref{eq:approximated_reverse_sde} to map the prior distribution $\pi$ to $p^{SDE}_{\theta}$.

\begin{gather}\label{eq:approximated_reverse_sde}
d\textbf{x}_t=[f(\textbf{x}_t,t)-G(\textbf{x}_t,t)G(\textbf{x}_t,t)^Ts_{\theta}(\textbf{x}_t,t)]dt + G(\textbf{x}_t,t)d\Bar{\textbf{w}_t},
\end{gather}If the prior distribution is close to the diffused distribution and the approximated score function is close to the ground truth score function, the modeled distribution  $p^{SDE}_{\theta}$ is provably close to the target distribution $p_0(\textbf{x}_0)$. This statement is formalised in the language of distributional distances in the work of \cite{song2021maximum}. 




\textbf{Training:} A neural network $s_\theta(\textbf{x}_t,t)$ can be trained to approximate the score function $\nabla_{\textbf{x}_t}{\ln{p_t(\textbf{x}_t)}}$ by minimizing the weighted score matching objective

\begin{gather}
\begin{aligned}
    \mathcal{L}_{SM}(\theta, \lambda(\cdot)) := 
    \frac{1}{2} \mathbb{E}_{\subalign{&t \sim U(0,T)\\ &\textbf{x}_t \sim p_t(\textbf{x}_t)}} [\lambda(t) \norm{\nabla_{\textbf{x}_t}{\ln{p_t(\textbf{x}_t)}} - s_\theta(\textbf{x}_t,t)}_2^2]
\end{aligned}
\end{gather}
where $\lambda: [0,T] \xrightarrow{} \mathbb{R}_+$ is a positive weighting function.

However, the above quantity cannot be optimized directly since we don't have access to the ground truth score $\nabla_{\textbf{x}_t}{\ln{p_t(\textbf{x}_t)}}$. Therefore in practice, a different objective has to be used \cite{score_matching, vincent2011connection, song2020score}. In \cite{song2020score}, the weighted denoising score-matching objective is used, which is defined as 

\begin{gather}\label{DSM for uniform diffusion models}
\begin{aligned}
    \mathcal{L}_{DSM}(\theta, \lambda(\cdot)) := 
    \frac{1}{2} \mathbb{E}_{\subalign{&t \sim U(0,T)\\ &\textbf{x}_0 \sim p_0(\textbf{x}_0) \\ &\textbf{x}_t \sim p_t(\textbf{x}_t | \textbf{x}_0)}} [\lambda(t) \norm{\nabla_{\textbf{x}_t}{\ln{p_t(\textbf{x}_t | \textbf{x}_0)}} - s_\theta(\textbf{x}_t,t)}_2^2]
\end{aligned}
\end{gather}

The difference between DSM and SM is the replacement of the ground truth score which we do not know by the score of the perturbation kernel which we know analytically for many choices of forward SDEs. The choice of the weighted DSM objective is justified because the weighted DSM objective is equal to the SM objective up to a constant that does not depend on the parameters of the model $\theta$. The reader can refer to \cite{vincent2011connection} for the proof. 

\section{Training details}
\label{sec:hparams}
We trained the score model using the weighted denoising score matching objective \cite{song2020score}, presented in eq. \ref{DSM for uniform diffusion models}. We used the likelihood weighting function, i.e. $\lambda(t)=g(t)^2$, where $g(t)$ is the diffusion coefficient of the forward SDE.

\begin{equation}\label{DSM for uniform diffusion models}
    \mathcal{L}_{DSM}(\theta, \lambda(\cdot)) := 
    \frac{1}{2} \mathbb{E}_{t \sim U(0,T)} \mathbb{E}_{\textbf{x}_0 \sim p_0(\textbf{x}_0)} \mathbb{E}_{\textbf{x}_t \sim p_t(\textbf{x}_t | \textbf{x}_0)} [\lambda(t) \norm{\nabla_{\textbf{x}_t}{\ln{p_t(\textbf{x}_t | \textbf{x}_0)}} - s_\theta(\textbf{x}_t,t)}_2^2]
\end{equation}

\subsection{Euclidean data}
For all of our experiments on Euclidean data, we used a fully connected network with 5 hidden layers and 2048 nodes in each hidden layer to approximate the score function. The input and output dimension is the same as the ambient dimension. For the optimisation of the model, we used the Adam algorithm with a learning rate of $2\mathrm{e}{-5}$ and exponential moving average (EMA) on the weights of the model with a decay rate of $0.9999$. Moreover, we chose the variance exploding SDE \cite{song2020score} as the forward process with $\sigma_{min}=0.01$ and $\sigma_{max}=4$.

\subsection{ Image data}    
For all our of our experiments on image data, we used the DDPM architecture \cite{ho2020denoising} with variance exploding SDE \cite{song2020score} and hyperparameters indicated in Table \ref{tab:model_params}.

\subsection{Auto-encoder}
The encoder and decoder encoder architectures are based on the DDPM U-Net \cite{ddpm}, which we call half-U nets.

For the encoder we used the downsampling part of the U-Net and removed the upsampling part and the skip connections. The downscaled tensor is flattened and mapped to the latent dimension with an additional linear layer.

For the decoder we start by linearly transforming the latent vector and reshaping it into a tensor of appropriate dimension. Then we used the upsampling part of the DDPM U-Net.

We used the Adam optimizer and EMA rate $0.999$. We used learning rate scheduler reducing the loss on plateau starting form $10^{-4}$ and stopping at $10^{-5}$. We trained the auto-encoder for each latent dimension for 36h on NVIDIA A-100 GPU. At the end we used checkpoints which minimized the validation loss to evaluate the reconstruction error.

All other hyperparameters are included in Table \ref{tab:model_params}.

\begin{table}[h]
\centering
\begin{tabular}{|l|l|l|}
\hline
\textbf{Hyper-parameter} & \textbf{MNIST} & \textbf{Synthetic Image data} \\
\hline
Number of filters & 128 & 128 \\
\hline
Channel multipliers & (1, 2, 2, 4) & (1, 2, 2, 2) \\
\hline
Dropout & 0.1 & 0.1 \\
\hline
EMA rate & 0.999 & 0.999 \\
\hline
Normalization & GroupNorm & GroupNorm \\
\hline
Nonlinearity & Swish & Swish \\
\hline
Number of residual blocks & 4 & 4 \\
\hline
Attention resolution & 16 & 16 \\
\hline
Convolution size & 3 & 3 \\
\hline
$\sigma_\text{min}$ & $0.009$ & $0.01$ \\
\hline
$\sigma_\text{max}$ & $50$ & $50$  \\
\hline
Learning rate & Scheduler($10^{-4}$, $10^{-5}$) & $2 \cdot 10^{-4}$ \\
\hline
\end{tabular}

\caption{DDPM Model Parameters}
\label{tab:model_params}
\end{table}

\section{Benchmarking}
\label{sec:benchmark}
We compared our method against well established approaches to intrinsic dimensionality estimation: the MLE estimator \cite{dim_MLE}, \cite{haro_mle}, Local PCA \cite{fan_local_pca} and Probabilistic PCA \cite{auto_ppca} \cite{ppca}. For MLE estimator and local PCA we used the implementation provided in the R package  \textsc{intrinsicDimension} \cite{R_intrinsic_dim}. The MLE estimator has an important hyperparameter $m$ - the number of nearest neighbour distances that should be used for the dimension estimation. We used values $m=5$ and $20$ since these are extremal values considered in \cite{pope2021intrinsic}. For PPCA we used the \textsc{scikit-learn} implementation \cite{sklearn}. The code for reproducing the benchmarking experiments is included in our codebase.

\section{Proofs}
\label{appendix:proof}
Here we provide full proofs for the statements in Section \ref{sec:theory}. First, we show that for any point $\textbf{x}$ sufficiently close to the data manifold and sufficiently small $t$ the score $\nabla_\textbf{x} \ln p_t(\textbf{x})$ points directly at the manifold. We demonstrate this by showing that projection of the score in any direction $\boldsymbol{\nu} \perp \textbf{n}$ vanishes in proportion to the projection on $\textbf{n} = \frac{(\pi(\textbf{x}) - \textbf{x})}{\norm{\pi(\textbf{x}) -\textbf{x}}}$ as $t \to 0$. Then Theorem \ref{thm:score_orthogonal} and Corollary \ref{cor:score_ratio} will follow easily from this result.

\begin{theorem}
\label{thm:master_thm}
Suppose that the the support of the data distribution $P_0$ is contained in a compact embedded sub-manifold $\mathcal{M} \subseteq \mathbb{R}^d$ and let $P_t$ be the distribution of samples from $P_0$ diffused for time $t$. Then, under mild assumptions, for any point $\textup{\textbf{x}} \in \mathbb{R}^d$ sufficiently close to $\mathcal{M}$, with orthogonal projection on $\mathcal{M}$, given by $\pi(\textup{\textbf{x}})$. Let $\textup{\textbf{n}}$ be a unit vector pointing from $\textup{\textbf{x}}$ to $\pi(\textup{\textbf{x}})$, then we have that for any unit vector $\boldsymbol{\nu}$ orthogonal to $\textup{\textbf{n}}$: 
\begin{gather*}
    \frac{\boldsymbol{\nu}^T \nabla_\textup{\textbf{x}} \ln p_t(\textup{\textbf{x}})}{\textup{\textbf{n}}^T \nabla_\textup{\textbf{x}} \ln p_t(\textup{\textbf{x}})} \to 0, \text{as } t \to 0.
\end{gather*}
\end{theorem}

\textbf{Assumptions}
\begin{enumerate}
    \item The distribution $P_0$ has a smooth density $p_0$ wrt the volume measure on the manifold.
    \item The density $p_0$ is bounded away from zero on the manifold.
\end{enumerate}

\subsection*{Illustrative simple case}

\begin{figure}
\centering
  \includegraphics[width=0.7\linewidth]{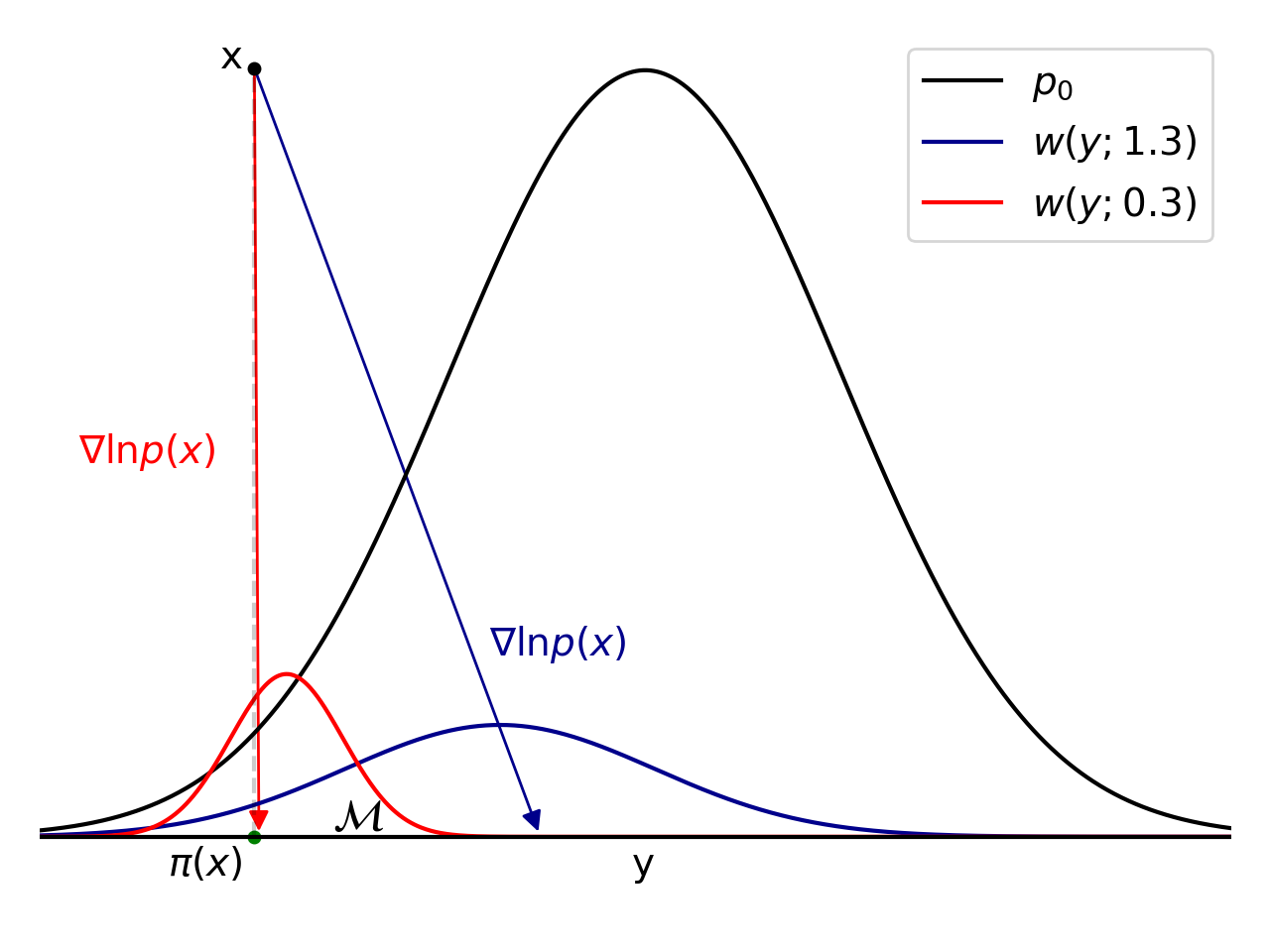}
  \caption[Caption without FN]{The score is the weighted average of vectors pointing from $\textbf{x}$ to $\textbf{y}$ with weights given by $w(\textbf{y};\sigma_t)$. As $\sigma_t$ decreases weights $w(\textbf{y}; \sigma_t)$ concentrate around $\pi(\textbf{x})$ and the influence of points $\textbf{y}$ away from projection $\pi(\textbf{x})$ becomes insignificant. Therefore, the direction of the score tends to a align with $\pi(\textbf{x})-\textbf{x}$. (Norm of the score vectors on the figure was scaled for better visibility. The direction is preserved.)}
  \label{fig:example}
\end{figure}

We first present an illustrative proof of a simple case of $\mathcal{M}$ being a linear subspace with $k=1$ and $d=2$. This case gives all of the essential ideas behind the general proof without much of the technicality. For the more interested reader, we then provide a proof of the result for a general manifold, using tools such as the notion of tubular neighbourhoods and some results from Morse theory. 
    
    Without the loss of generality assume that $\mathcal{M} = \{ (x_1,x_2) \in \mathbb{R}^2 : x_2 = 0 \}$ the line given by the $x_1$-axis. Pick a point $\textbf{x} \in \mathbb{R}^2$. The score at point $\textbf{x}$ is given by
\begin{gather*}
    \nabla_\textbf{x} \ln p_t(\textbf{x}) = \frac{1}{\sigma_t^2 p_t(\textbf{x})} \int_\mathcal{M} (\textbf{y}-\textbf{x}) \mathcal{N}(\textbf{y} | \textbf{x}, \sigma^2_t\textbf{I}) p_0(\textbf{y}) d\textbf{y}.
\end{gather*}
Notice that $\mathcal{N}((y_1, y_1) | (x_1, x_2), \sigma^2_t\textbf{I})$\footnote{In component-wise notation $\textbf{y} = (y_1, y_2)$ and $\textbf{x} = (x_1, x_2)$.} is a bivariate normal distribution and its restriction to $\mathcal{M}$ is equal to $\mathcal{N}((y_1, 0) | (x_1, x_2), \sigma^2_t\textbf{I}) = \mathcal{N}(0 | x_2, \sigma^2_t) \mathcal{N}(y_1 | x_1, \sigma^2_t)$. \vspace{0.2cm} Therefore
\begin{gather*}
    \nabla_\textbf{x} \ln p_t(\textbf{x}) = \frac{\mathcal{N}(0 | x_2, \sigma^2_t)}{\sigma_t^2p_t(\textbf{x})} \int_\mathcal{M} (\textbf{y}-\textbf{x})  \mathcal{N}(y_1 | x_1, \sigma^2_t) p_0(\textbf{y}) d\textbf{y}.
\end{gather*}
This means that the score is the weighted average of vectors pointing from $\textbf{x}$ to $\textbf{y}$ over all choices of points $\textbf{y}$ on the manifold, with weights given by $w(\textbf{y};\sigma_t) :=\mathcal{N}(y_1 | x_1, \sigma^2_t) p_0(\textbf{y})$ (see Figure \ref{fig:example} for visual explanation). For small $\sigma_t$ these weights concentrate around $\pi(\textbf{x})=(x_1,0)$ the projection of $\textbf{x}$ on $\mathcal{M}$, and vanishing far away from it. Consider a ratio of the tangential part to the normal part of the score:
\begin{gather*}
    \frac{\boldsymbol{\nu}^T \nabla_\textbf{x} \ln p_t(\textbf{x})}{\textbf{n}^T \nabla_\textbf{x} \ln p_t(\textbf{x})} = \frac{\int_\mathcal{M} \boldsymbol{\nu}^T(\textbf{y}-\textbf{x})  \mathcal{N}(y_1 | x_1, \sigma^2_t) p_0(\textbf{y}) d\textbf{y}}{\int_\mathcal{M} \textbf{n}^T(\textbf{y}-\textbf{x})  \mathcal{N}(y_1 | x_1, \sigma^2_t) p_0(\textbf{y}) d\textbf{y}} 
    \xrightarrow[\sigma_t \to 0]{}
    \frac{\int_\mathcal{M} \boldsymbol{\nu}^T(\textbf{y}-\textbf{x})  \delta_{x_1}(y_1) p_0(\textbf{y}) d\textbf{y}}{\int_\mathcal{M} \textbf{n}^T(\textbf{y}-\textbf{x})  \delta_{x_1}(y_1) p_0(\textbf{y}) d\textbf{y}} \\
    = \frac{\boldsymbol{\nu}^T ((x_1,0)-\textbf{x})}{\textbf{n}^T ((x_1,0)-\textbf{x})} = \frac{(1,0)^T (0,-x_2)}{(0,1)^T (0,-x_2)} = 0
\end{gather*}
where $\boldsymbol{\nu}$ and $\textbf{n}$ are unit vectors in tangential and normal directions respectively. This implies,
\begin{gather*}
    \text{S}_{\cos} (\textbf{n}, \nabla_\textbf{x} \ln p_t(\textbf{x})) 
    = \frac{\textbf{n}^T  \nabla_\textbf{x} \ln p_t(\textbf{x})}{\norm{ \nabla_\textbf{x} \ln p_t(\textbf{x})}} 
    = \frac{\textbf{n}^T  \nabla_\textbf{x} \ln p_t(\textbf{x})}{\sqrt{(\textbf{n}^T  \nabla_\textbf{x} \ln p_t(\textbf{x}))^2 + (\boldsymbol{\nu}^T  \nabla_\textbf{x} \ln p_t(\textbf{x}))^2}} 
    \\ = \frac{1}{\sqrt{1 + \big( \frac{\boldsymbol{\nu}^T\nabla_\textbf{x} \ln p_t(\textbf{x})}{\textbf{n}^T  \nabla_\textbf{x} \ln p_t(\textbf{x})} \big)^2}} 
    \xrightarrow[t \to 0]{} 1.
\end{gather*}
This establishes the theorem for the simple case. The corollary follows immediately since in the simple case we have $\mathbf{T} = \boldsymbol{\nu}^T$ and $\mathbf{N}=\textbf{n}^T$. \qed

\subsection*{Deriving the formula for the density of $P_t$}
Let $\mathcal{M}$ be a compact $k$-dimensional manifold embedded in $\mathbb{R}^d$.
Let $A \subseteq \mathbb{R}^d$. We define the measure $P_0$ on $\mathbb{R}^d$ as
\begin{gather}
    P_0(A) := \int_{A \cap M} p_0(\textbf{y}) d\textbf{y}
\end{gather}
where $d\textbf{y}$ is the volume form\footnote{Integrating over $A$ using the volume form of $\mathcal{M}$ can be thought of as taking an appropriately re-scaled Lebesgue integral over $A$. That is $P_0(A) = \int_{A \cap \mathcal{M}} p_0(\textbf{y})d\textbf{y} = \int_A \int_{\mathbb{R}^d} \delta (\textbf{s}-\textbf{y})\hat{p}_0(\textbf{s})d\textbf{s} d\textbf{y}$. Where the latter are Lebesgue integrals and $\hat{p}_0(\textbf{s}) = p_0(\textbf{s}) \text{ for } \textbf{s}\in\mathcal{M}$ and zero otherwise.} on $M$ and $p_0$ is a smooth function on $\mathcal{M}$ \footnote{i.e. for any chart $\phi: \mathbb{R}^k \supseteq U \xrightarrow[]{} M$ the composition $p_0 \circ \phi: \mathbb{R}^k \supseteq U \xrightarrow[]{} \mathbb{R}$ is smooth.} such that $\int_\mathcal{M} p_0(\textbf{y}) d\textbf{y} = 1$. Let $f : \mathbb{R}^d \rightarrow \mathbb{R}$ be a $P_0$-measurable function. By approximating $f$ with simple functions (linear combination of indicator functions) we conclude that:
\begin{gather}
\label{eq:itegrals_on_M}
    \int_A f dP_0 = \int_{A \cap M} f(\textbf{y}) p_0(\textbf{y}) d\textbf{y}
\end{gather}
Consider a measure $P_t$ as a convolution of $P_0$ with a normal distribution on $\mathbb{R}^d$. For any measurable  $A \subseteq \mathbb{R}^d$ we have
\begin{gather*}
    (P_0 \ast  \mathcal{N}_{0, t})(A) := \int_{\mathbb{R}^d}\int_{A-\textbf{y}}d\mathcal{N}_{0, t}(\textbf{x})dP_0(\textbf{y}) = \int_{\mathbb{R}^d}\int_{A-\textbf{y}}\mathcal{N}(\textbf{x} | 0, \sigma^2_t\textbf{I})d\textbf{x} dP_0(\textbf{y}) \\
    =  \int_{\mathbb{R}^d}\int_{A}\mathcal{N}(\textbf{x}-\textbf{y} | 0, \sigma^2_t\textbf{I})d\textbf{x} dP_0(\textbf{y}) =  \int_{\mathbb{R}^d}\int_{A}\mathcal{N}(\textbf{y} | \textbf{x}, \sigma^2_t\textbf{I})d\textbf{x} dP_0(\textbf{y}) \\
    = \int_{A} \int_{\mathbb{R}^d} \mathcal{N}(\textbf{y} | \textbf{x}, \sigma^2_t\textbf{I})dP_0(\textbf{y}) d\textbf{x} \overset{\eqref{eq:itegrals_on_M}}{=} \int_{A} \int_\mathcal{M} \mathcal{N}(\textbf{y} | \textbf{x}, \sigma^2_t\textbf{I}) p_0(\textbf{y}) d\textbf{y} d\textbf{x}
\end{gather*}
where $d\textbf{y}$ is a volume form on $M$ and $d\textbf{x}$ is a volume form on $\mathbb{R}^d$.
Therefore the measure $P_t$ has a density on $\mathbb{R}^d$ given by:
\begin{gather}
    \label{convolution}
    p_t(\textbf{x}) = \int_\mathcal{M} \mathcal{N}(\textbf{y} | \textbf{x}, \sigma^2_t\textbf{I}) p_0(\textbf{y}) d\textbf{y}.
\end{gather}
Note the $\mathcal{N}(\textbf{y}|\textbf{x},\sigma^2_t\textbf{I})$ here. Typically one would write this as $\mathcal{N}(\textbf{x}|\textbf{y},\sigma^2_t\textbf{I})$ and think of \eqref{convolution} as the quantity of probability mass at point $\textbf{x}$ after diffusing for time $t$ with initial distribution $p_0(\textbf{y})$. We instead write it this way as it will be more intuitive to think of \eqref{convolution} as the average probability mass that intersects the manifold after diffusing from a delta distribution at $\textbf{x}$ (where the average is taken over $p_0(\textbf{y})$). These are of course equivalent as $\mathcal{N}(\textbf{x}|\textbf{y},\sigma^2_t\textbf{I})$ is symmetric is $\textbf{x}$ and $\textbf{y}$.

\subsection*{Tubular Neighbourhoods}
First we need to ensure that the point $\textbf{x}$ has a unique projection on $\mathcal{M}$. This is always true for an $\textbf{x}$ sufficiently close to $\mathcal{M}$. We can formalize this with the notion of \textit{tubular neighbourhood} - a tube around $\mathcal{M}$ such that every point $\textbf{x}$ inside can be uniquely represented as a sum of the point on the manifold and a vector from the normal bundle i.e. $\textbf{x} = \textbf{y} + \textbf{v}$ where $\textbf{y} \in \mathcal{M}$ and $\textbf{v} \in \mathcal{N}_\textbf{y} \mathcal{M}$. Formally:

\begin{definition}
Endpoint Map\\
The endpoint map $Y: \mathcal{NM} \xrightarrow[]{} \mathbb{R}^d$ is defined by $Y(\textbf{y}, \textbf{v}) = \textbf{y} + \textbf{v}$ for $\textbf{y} \in \mathcal{M}$ and $\textbf{v} \in \mathcal{N}_\textbf{y} \mathcal{M}$. 
\end{definition}

\begin{definition}
Tubular Neighbourhood\\
A (uniform) tubular neighbourhood of $\mathcal{M}$ is a neighbourhood $U_R$ of $\mathcal{M}$ in $\mathbb{R}^d$ that is a diffeomorphic image under the endpoint map of an open subset $V_R \subseteq \mathcal{NM}$ of the form:
\begin{gather*}
    V_R = \{ (\textbf{y},\textbf{v}) \in \mathcal{NM}: \norm{\textbf{v}}_2 < R \} 
\end{gather*}
\end{definition}
Since $Y$ restricted to $V_R$ is a diffeomorphism\footnote{so in particular it is a bijection}, it follows  that every point $\textbf{x} = (\textbf{y},\textbf{v})$ in the tubular neighbourhood has a unique orthogonal projection on $\mathcal{M}$ given by $\textbf{y}$. We will denote this projection as $\pi(\textbf{x})$.

Conveniently, it turns out that every  compact embedded submanifold of $\mathbb{R}^d$ has a tubular neighborhood.

\begin{theorem}
Tubular Neighborhood Theorem \cite[Theorem 5.25]{lee2019_riemman} \\  
Every  compact  embedded submanifold of $\mathbb{R}^d$ has a uniform tubular neighborhood.
\end{theorem}

\subsection*{Preliminary lemmas and  Morse theory}
In this section we will establish that for every $\textbf{x}$ in the tubular neighbourhood of $\mathcal{M}$ there exists an open neighbourhood $E$ of $\pi(\textbf{x})$ such that:
\begin{enumerate}
    \item $\textbf{n}^T(\textbf{y}-\textbf{x}) > \norm{\pi(\textbf{x}) - \textbf{x}}_2 - \varepsilon $ on $E$ .
    \item $|\boldsymbol{\nu}^T(\textbf{y}-\textbf{x})| < \varepsilon$ on $E$.
    \item 
    The mass of a Gaussian centred at $\textbf{x}$ is concentrated in $E$,
    $$ \frac{\int_{\mathcal{M} \setminus E}\mathcal{N}(\textbf{y} | \textbf{x}, \sigma^2_t\textbf{I}) d\textbf{y}}{ \int_E \mathcal{N}(\textbf{y} | \textbf{x}, \sigma^2_t\textbf{I}) d\textbf{y}} \to 0 \text{ as } t \to 0. $$
\end{enumerate}

We begin by defining an $E$ which satisfies the first two conditions.
\begin{lemma}
\label{lemma:small_projection}
Choose $E$ contained in a ball of radius $0 < \varepsilon < \norm{\textup{\textbf{x}} - \pi(\textup{\textbf{x}})}$ around $\pi(\textup{\textbf{x}})$. Let $\textup{\textbf{y}} \in E$, and let $ \textup{\textbf{v}}_\varepsilon := \textup{\textbf{y}} - \pi(\textup{\textbf{x}}) $. Then 
    \begin{enumerate}
        \item $\textup{\textbf{n}}^T(\textup{\textbf{y}}-\textup{\textbf{x}}) > \norm{\pi(\textup{\textbf{x}}) - \textup{\textbf{x}}}_2 - \varepsilon $ on $E$ .
        \item $|\boldsymbol{\nu}^T(\textup{\textbf{y}}-\textup{\textbf{x}})| < \varepsilon$ on $E$.
    \end{enumerate}
\end{lemma}
\begin{proof}
By direct computation
\begin{align*}
   \textbf{n}^T(\textbf{y}-\textbf{x}) &= \textbf{n}^T( (\textbf{x} - \pi(\textbf{x})) + \textbf{v}_\varepsilon)\\
   &= \textbf{n}^T(\textbf{x} - \pi(\textbf{x})) + \textbf{n}^T\textbf{v}_\varepsilon \\
   &= \norm{\textbf{x} - \pi(\textbf{x})}_2 + \textbf{n}^T\textbf{v}_\varepsilon\\
   &\geq  \norm{\textbf{x} - \pi(\textbf{x})} - \norm{\textbf{v}_\varepsilon}_2.
\end{align*}   
We have that $\norm{\textbf{v}_\varepsilon} < \varepsilon$, hence for all $\textbf{y}$ in $E$ we have that $ \textbf{n}^T(\textbf{y}-\textbf{x}) \geq \norm{\textbf{x} - \pi(\textbf{x})} - \varepsilon > 0 $. For the second inequality \begin{align*}
    | \boldsymbol{\nu}^T(\textbf{y}-\textbf{x}) | &\leq | \boldsymbol{\nu}^T( (\textbf{x} - \pi(\textbf{x}))| + |\boldsymbol{\nu}^T\textbf{v}_\varepsilon|  \\
    &\leq  \norm{\textbf{v}_\varepsilon}_2 \\ 
    &\leq \varepsilon.
\end{align*}
\end{proof}

Now to find $E$ which also satisfies the last condition we proceed by recalling some elementary definitions and results of Morse theory.

\begin{theorem}
Morse lemma \cite[Corollary 1.17]{nicolaescu2011morse_theory} \\
If $\textup{\textbf{y}}_0$ is a non-degenerate critical point of index $\gamma$ of a smooth function $f : \mathcal{M} \xrightarrow[]{} \mathbb{R}$, then there exist a chart $\phi = (\phi_i)_{i=1}^k$  in a neighbourhood $U$ of $\textup{\textbf{y}}_0$ such that $\phi(\textbf{y}_0) = 0$, and in this chart we have the equality:
$$ f(\textup{\textbf{y}}) = f(\textup{\textbf{y}}_0) - \sum_{i=1}^\gamma \phi_i(\textup{\textbf{y}})^2 + \sum_{i=\gamma + 1}^k \phi_i(\textup{\textbf{y}})^2 $$
\end{theorem}

Let $f_\textup{\textbf{x}}(\textbf{y}): \mathcal{M} \xrightarrow[]{} \mathbb{R}$  denote the squared distance function from $\textbf{x}$ given by $f_\textup{\textbf{x}}(\textbf{y}) = \norm{\textbf{x} - \textbf{y}}_2^2$.
We will establish that if $\textbf{x}$ is in a tubular neighbourhood, then its projection $\pi(\textbf{x})$ is a non-degenerate critical point of $f_\textup{\textbf{x}}$ of index zero.

\begin{definition}
Focal point \\
A point $\textbf{x} =Y(\textbf{y}, \textbf{v})$ in the image of the endpoint map $Y$, is called a non-focal point of $\mathcal{M}$ with respect to $\textbf{y}$ if $dY(\textbf{y},\textbf{v})$ is an isomorphism. Otherwise it is called a focal point.
\end{definition}

\begin{theorem}
Critical points and focal points \cite[Theorem 4.2.6]{palais1988critical} \\
Let $\mathcal{M}$ be an embedded submanifold of $\mathbb{R}^d$, $\textup{\textbf{y}} \in M$, $\textup{\textbf{v}} \in \mathcal{N}_\textup{\textbf{y}}\mathcal{M}$, and $\textup{\textbf{x}} = Y (\textup{\textbf{y}},\textup{\textbf{v}}) = \textup{\textbf{y}} + \textup{\textbf{v}}$. Then
\begin{enumerate}
    \item $\textup{\textbf{y}}$ is a critical point of $f_\textup{\textbf{x}}$.
    \item $\textup{\textbf{y}}$ is a non-degenerate critical point of $f_\textup{\textbf{x}}$ if and only if $\textup{\textbf{x}}$ is a non-focal point.
    \item Index of $f_a$ at $\textbf{y}$ is equal to the number of focal points of $\mathcal{M}$ with respect to $\textup{\textbf{y}}$ on the line segment joining $\textup{\textbf{y}}$ to $\textup{\textbf{x}}$.
\end{enumerate}
\end{theorem}
Because the restriction of the endpoint map $Y$ to the tubular neighbourhood is a diffeomorphism, the differential $dY$ is an isomorphism for every point in the tubular neighbourhood. Therefore there are no focal points of $\mathcal{M}$ in the tubular neighbourhood. Hence, it follows directly from the above theorem, that if $\textbf{x}$ is in the tubular neighbourhood, then the projection $\pi(\textbf{x})$ is a non-degenerate critical point of $f_\textup{\textbf{x}}$ of index zero. Now we are ready to prove the following lemma.

\begin{lemma}
\label{lemma:gaussian_concentration}
    There exists a connected open neighbourhood $E$ of $\pi(\textup{\textbf{x}})$ satisfying conditions of lemma \ref{lemma:small_projection} and such that, 
    \begin{gather}
        \frac{\int_{\mathcal{M} \setminus E}\mathcal{N}(\textup{\textbf{y}} | \textup{\textbf{x}}, \sigma^2_t\textup{\textbf{I}}) d\textup{\textbf{y}}}{ \int_E \mathcal{N}(\textup{\textbf{y}} | \textup{\textbf{x}}, \sigma^2_t\textup{\textbf{I}}) d\textup{\textbf{y}}} \to 0 \text{ as } t \to 0. 
    \end{gather}
\end{lemma}
\begin{proof}
Fix $\varepsilon > 0$. Then conditions of lemma \ref{lemma:small_projection} are satisfied inside $B(\pi(\textbf{x}), \varepsilon)$. We have demonstrated that $\pi(\textbf{x})$ fulfills the criteria stipulated by the Morse lemma. Consequently, we can pick $\tilde{U}$ as the neighborhood and $\phi$ as the coordinate system that the Morse lemma provides. Now let $U = \tilde{U} \cap B(\pi(\textbf{x}), \varepsilon)$. Since $\mathcal{M}$ is compact and $U$ is open, there exists $m = \min_{\mathcal{M} \setminus U} f_\textup{\textbf{x}}(\textbf{y})$ and by uniqueness of projection $f_\textup{\textbf{x}}(\pi(\textbf{x})) < m$. 
Let $r= \sqrt{m - (f_\textup{\textbf{x}}(\pi(\textbf{x}))) / 2}$ and let $E = \phi^{-1}(B(0, r))$. For all $\textbf{y} \in E$ we have:
$$f_\textup{\textbf{x}}(\textbf{y}) = f_\textup{\textbf{x}}(\pi(\textbf{x})) + \norm{\phi(\textbf{y})}^2 <   f_\textup{\textbf{x}}(\pi(\textbf{x})) + r^2 < m.$$ 
Notice that for every $\textbf{y} \in U \setminus E$ we have $f_\textup{\textbf{x}}(\textbf{y}) \geq f_\textup{\textbf{x}}(\pi(\textbf{x})) + r^2$. Therefore we have established that
\begin{gather}
\label{eq:less}
   \forall_{\textbf{y} \in E} \forall_{\tilde{\textbf{y}} \in \mathcal{M} \setminus E } f_\textup{\textbf{x}}(\textbf{y}) < f_\textup{\textbf{x}}(\tilde{\textbf{y}}). 
\end{gather}
Computing directly, we have   
\begin{gather}
\label{eq:normals2}
    \frac{\int_{\mathcal{M} \setminus E}\mathcal{N}(\textbf{y} | \textbf{x}, \sigma^2_t\textbf{I}) d\textbf{y}}{ \int_E \mathcal{N}(\textbf{y} | \textbf{x}, \sigma^2_t\textbf{I}) d\textbf{y}}
    = \frac{\int_{\mathcal{M} \setminus E} \exp\{{-f_\textup{\textbf{x}}(\textbf{y})/2\sigma_t^2}\} d\textbf{y}}{ \int_E \exp \{{-f_\textup{\textbf{x}}(\textbf{y})/2\sigma_t^2} \}d\textbf{y}} 
\end{gather}
By the mean value theorem there exists $\textbf{y}^*\in E$ and $\tilde{\textbf{y}}^*\in \mathcal{M} \setminus E$ such that 
\begin{align*}
    \int_E \exp \{{-f_\textup{\textbf{x}}(\textbf{y})/2\sigma_t^2} \}d\textbf{y} &= \text{Vol}(E)\exp\{-f_\textup{\textbf{x}}(\textbf{y}^*)/2\sigma_t^2 \} \\
    \int_{\mathcal{M} \setminus E} \exp \{{-f_\textup{\textbf{x}}(\textbf{y})/2\sigma_t^2} \}d\textbf{y} &= \text{Vol}(\mathcal{M} \setminus E)\exp\{-f_\textup{\textbf{x}}(\tilde{\textbf{y}}^*)/2\sigma_t^2 \}
\end{align*}
We can use this to evaluate (\ref{eq:normals2}) to give,
\begin{gather*}
    \frac{\int_{\mathcal{M} \setminus E}\mathcal{N}(\textbf{y} | \textbf{x}, \sigma^2_t\textbf{I}) d\textbf{y}}{ \int_E \mathcal{N}(\textbf{y} | \textbf{x}, \sigma^2_t\textbf{I}) d\textbf{y}} 
    = \frac{\text{Vol}(\mathcal{M} \setminus E) \exp \{{-f_\textup{\textbf{x}}(\tilde{\textbf{y}}^*)/2\sigma_t^2} \}}{\text{Vol}(E)\exp\{-f_\textup{\textbf{x}}(\textbf{y}^*)/2\sigma_t^2\} }
    =  \frac{\text{Vol}(\mathcal{M} \setminus E) }{\text{Vol}(E)} \exp\left\{-\frac{f_\textup{\textbf{x}}(\tilde{\textbf{y}}^*) - f_\textup{\textbf{x}}(\textbf{y}^*)}{2 \sigma_t^2}   \right\}
\end{gather*}
Since by (\ref{eq:less}) $f_\textup{\textbf{x}}(\tilde{\textbf{y}}^*) - f_\textup{\textbf{x}}(\textbf{y}^*) > 0$  the above goes to zero as $\sigma_t$ goes to zero. Moreover, since $E \subseteq U \subseteq B(\pi(\textbf{x}), \varepsilon)$, the conditions of lemma \ref{lemma:small_projection} are also satisfied.
\end{proof}

\subsection*{Proof of Theorem \ref{thm:master_thm}}
Fix $\varepsilon > 0$. Assume that $\textbf{x}$ is in a tubular neighbourhood of $\mathcal{M}$, so that the projection $\pi(\textbf{x})$ exists.
\begin{gather}
\label{eq:begin_proof}
    \frac{\boldsymbol{\nu}^T \nabla_\textbf{x} \ln p_t(\textbf{x})}{\textbf{n}^T \nabla_\textbf{x} \ln p_t(\textbf{x})}   = \frac{\boldsymbol{\nu}^T \nabla_\textbf{x}  p_t(\textbf{x})}{\textbf{n}^T \nabla_\textbf{x}  p_t(\textbf{x})} = \frac{\int_\mathcal{M} \boldsymbol{\nu}^T (\textbf{y}-\textbf{x}) \mathcal{N}(\textbf{y} | \textbf{x}, \sigma^2_t\textbf{I}) p_0(\textbf{y}) d\textbf{y}}{\int_\mathcal{M} \textbf{n}^T (\textbf{y}-\textbf{x}) \mathcal{N}(\textbf{y} | \textbf{x}, \sigma^2_t\textbf{I}) p_0(\textbf{y}) d\textbf{y}}
\end{gather}
Split $\mathcal{M}$ into two parts: $E$ and $\mathcal{M} \setminus E$, where $E=B(\textbf{x},r)\cap \mathcal{M}$ is an open neighbourhood of $\pi(\textbf{x})$ in $\mathcal{M}$ satisfying the conditions of Lemma \ref{lemma:gaussian_concentration} and Lemma \ref{lemma:small_projection} for a chosen $\varepsilon>0$, then we have that \eqref{eq:begin_proof} is equal to
\begin{gather}
\label{eq:two_terms}
   \frac{\int_{E} \boldsymbol{\nu}^T (\textbf{y}-\textbf{x}) \mathcal{N}(\textbf{y} | \textbf{x}, \sigma^2_t\textbf{I}) p_0(\textbf{y}) d\textbf{y}}{\int_\mathcal{M} \textbf{n}^T (\textbf{y}-\textbf{x}) \mathcal{N}(\textbf{y} | \textbf{x}, \sigma^2_t\textbf{I}) p_0(\textbf{y}) d\textbf{y}} + \frac{\int_{\mathcal{M} \setminus E} \boldsymbol{\nu}^T (\textbf{y}-\textbf{x}) \mathcal{N}(\textbf{y} | \textbf{x}, \sigma^2_t\textbf{I}) p_0(\textbf{y}) d\textbf{y}}{\int_\mathcal{M} \textbf{n}^T (\textbf{y}-\textbf{x}) \mathcal{N}(\textbf{y} | \textbf{x}, \sigma^2_t\textbf{I}) p_0(\textbf{y}) d\textbf{y}}.
\end{gather}
We begin by bounding the first term:
\begin{gather*}
    \frac{\int_{E} \boldsymbol{\nu}^T (\textbf{y}-\textbf{x}) \mathcal{N}(\textbf{y} | \textbf{x}, \sigma^2_t\textbf{I}) p_0(\textbf{y}) d\textbf{y}}{\int_\mathcal{M} \textbf{n}^T (\textbf{y}-\textbf{x}) \mathcal{N}(\textbf{y} | \textbf{x}, \sigma^2_t\textbf{I}) p_0(\textbf{y}) d\textbf{y}} 
    = \frac{\int_{E} \boldsymbol{\nu}^T (\textbf{y}-\textbf{x}) \mathcal{N}(\textbf{y} | \textbf{x}, \sigma^2_t\textbf{I}) p_0(\textbf{y}) d\textbf{y}}{\int_E \textbf{n}^T (\textbf{y}-\textbf{x}) \mathcal{N}(\textbf{y} | \textbf{x}, \sigma^2_t\textbf{I}) p_0(\textbf{y}) d\textbf{y} + \int_{\mathcal{M} \setminus E} \textbf{n}^T (\textbf{y}-\textbf{x}) \mathcal{N}(\textbf{y} | \textbf{x}, \sigma^2_t\textbf{I}) p_0(\textbf{y}) d\textbf{y}} \\
    = \frac{ \frac{\int_{E} \boldsymbol{\nu}^T (\textbf{y}-\textbf{x}) \mathcal{N}(\textbf{y} | \textbf{x}, \sigma^2_t\textbf{I}) p_0(\textbf{y}) d\textbf{y}}{\int_{E} \textbf{n}^T (\textbf{y}-\textbf{x})\mathcal{N}(\textbf{y} | \textbf{x}, \sigma^2_t\textbf{I}) p_0(\textbf{y}) d\textbf{y}}  }{1 + \frac{ \int_{\mathcal{M} \setminus E} \textbf{n}^T (\textbf{y}-\textbf{x})\mathcal{N}(\textbf{y} | \textbf{x}, \sigma^2_t\textbf{I}) p_0(\textbf{y}) d\textbf{y}}{\int_{E} \textbf{n}^T (\textbf{y}-\textbf{x})\mathcal{N}(\textbf{y} | \textbf{x}, \sigma^2_t\textbf{I}) p_0(\textbf{y}) d\textbf{y}} }
    =: \frac{A_t}{1+B_t}
\end{gather*}
For $A_t$ we have
\begin{gather*}
    |A_t| = \bigg| \frac{\int_{E} \boldsymbol{\nu}^T (\textbf{y}-\textbf{x}) \mathcal{N}(\textbf{y} | \textbf{x}, \sigma^2_t\textbf{I}) p_0(\textbf{y}) d\textbf{y}}{\int_{E} \textbf{n}^T (\textbf{y}-\textbf{x})\mathcal{N}(\textbf{y} | \textbf{x}, \sigma^2_t\textbf{I}) p_0(\textbf{y}) d\textbf{y}} \bigg|. 
\end{gather*}
Using the fact that $\textbf{n}^T (\textbf{y}-\textbf{x})$ is positive and bounded away from zero and applying the triangle inequality we obtain
\begin{align*}
    |A_t| &\leq \frac{\int_{E} |\boldsymbol{\nu}^T (\textbf{y}-\textbf{x})| \mathcal{N}(\textbf{y} | \textbf{x}, \sigma^2_t\textbf{I}) p_0(\textbf{y}) d\textbf{y}}{\int_{E} \textbf{n}^T (\textbf{y}-\textbf{x})\mathcal{N}(\textbf{y} | \textbf{x}, \sigma^2_t\textbf{I}) p_0(\textbf{y}) d\textbf{y}} \\
        &\leq \frac{ \varepsilon p_\text{max} \int_{E} \mathcal{N}(\textbf{y} | \textbf{x}, \sigma^2_t\textbf{I}) d\textbf{y}}{ (\norm{\pi(\textbf{x}) - \textbf{x}} - \varepsilon ) p_\text{min}\int_{E}
    \mathcal{N}(\textbf{y} | \textbf{x}, \sigma^2_t\textbf{I}) d\textbf{y}} \\
    &= \frac{p_\text{max}}{p_\text{min}} \frac{ \varepsilon}{\norm{\pi(\textbf{x}) - \textbf{x}} - \varepsilon },
\end{align*}
where in the second inequality we used that $0<p_\text{min} < p_0(\textbf{y}) < p_\text{max}$ , $|\boldsymbol{\nu}^T(\textbf{y}-\textbf{x})| < \varepsilon$ and $\textbf{n}^T(\textbf{y}-\textbf{x}) > \norm{\pi(\textbf{x}) - \textbf{x}} - \varepsilon $ on $E$.
Since $\varepsilon$ was arbitrary this term can be made arbitrary small.
Now we move to $B_t$. Let $D = \max_{\textbf{y} \in \mathcal{M}} \norm{\textbf{x} - \textbf{y}}$, by the triangle and Cauchy Schwarz inequalities we have
\begin{align*}
    |B_t| &\leq \frac{ \int_{\mathcal{M} \setminus E} |\textbf{n}^T (\textbf{y}-\textbf{x})|\mathcal{N}(\textbf{y} | \textbf{x}, \sigma^2_t\textbf{I}) p_0(\textbf{y}) d\textbf{y}}{\int_{E} \textbf{n}^T (\textbf{y}-\textbf{x})\mathcal{N}(\textbf{y} | \textbf{x}, \sigma^2_t\textbf{I}) p_0(\textbf{y}) d\textbf{y}} \\
    &\leq \frac{ p_\text{max} D}{ p_\text{min} (\norm{\pi(\textbf{x}) - \textbf{x}} - \varepsilon)} \frac{\int_{\mathcal{M} \setminus E}  \mathcal{N}(\textbf{y} | \textbf{x}, \sigma^2_t\textbf{I}) d\textbf{y}}{\int_{E}\mathcal{N}(\textbf{y} | \textbf{x}, \sigma^2_t\textbf{I}) d\textbf{y}}.
\end{align*}
which goes to zero as $t$ goes to zero. Finally we move to the second term in (\ref{eq:two_terms}), we start with the same steps as with the first term
\begin{gather*}
    \frac{\int_{\mathcal{M} \setminus E} \boldsymbol{\nu}^T (\textbf{y}-\textbf{x}) \mathcal{N}(\textbf{y} | \textbf{x}, \sigma^2_t\textbf{I}) p_0(\textbf{y}) d\textbf{y}}{\int_\mathcal{M} \textbf{n}^T (\textbf{y}-\textbf{x}) \mathcal{N}(\textbf{y} | \textbf{x}, \sigma^2_t\textbf{I}) p_0(\textbf{y}) d\textbf{y}}
    = \frac{ \frac{\int_{\mathcal{M} \setminus E} \boldsymbol{\nu}^T (\textbf{y}-\textbf{x}) \mathcal{N}(\textbf{y} | \textbf{x}, \sigma^2_t\textbf{I}) p_0(\textbf{y}) d\textbf{y}}{\int_{E} \textbf{n}^T (\textbf{y}-\textbf{x})\mathcal{N}(\textbf{y} | \textbf{x}, \sigma^2_t\textbf{I}) p_0(\textbf{y}) d\textbf{y}}  }{1 + \frac{ \int_{\mathcal{M} \setminus E} \textbf{n}^T (\textbf{y}-\textbf{x})\mathcal{N}(\textbf{y} | \textbf{x}, \sigma^2_t\textbf{I}) p_0(\textbf{y}) d\textbf{y}}{\int_{E} \textbf{n}^T (\textbf{y}-\textbf{x})\mathcal{N}(\textbf{y} | \textbf{x}, \sigma^2_t\textbf{I}) p_0(\textbf{y}) d\textbf{y}} } 
    =: \frac{C_t}{1 + B_t}
\end{gather*}
so we only need to examine $C_t$. By the triangle and Cauchy Schwarz inequalities again, we have
\begin{align*}
    |C_t| &\leq \frac{\int_{\mathcal{M} \setminus E} |\boldsymbol{\nu}^T (\textbf{y}-\textbf{x})| \mathcal{N}(\textbf{y} | \textbf{x}, \sigma^2_t\textbf{I}) p_0(\textbf{y}) d\textbf{y}}{\int_{E} \textbf{n}^T (\textbf{y}-\textbf{x})\mathcal{N}(\textbf{y} | \textbf{x}, \sigma^2_t\textbf{I}) p_0(\textbf{y}) d\textbf{y}} \\
    &\leq \frac{ p_\text{max} D}{ p_\text{min} (\norm{\pi(\textbf{x}) - \textbf{x}} - \varepsilon)} \frac{\int_{\mathcal{M} \setminus E}  \mathcal{N}(\textbf{y} | \textbf{x}, \sigma^2_t\textbf{I}) d\textbf{y}}{\int_{E}\mathcal{N}(\textbf{y} | \textbf{x}, \sigma^2_t\textbf{I}) d\textbf{y}}
\end{align*}
which goes to zero as $t$ goes to zero. 
Putting all together
\begin{align*}
    \frac{|\boldsymbol{\nu}^T \nabla_\textbf{x} \ln p_t(\textbf{x})|}{|\textbf{n}^T \nabla_\textbf{x} \ln p_t(\textbf{x})|} 
    &= \frac{|A_t + C_t|}{|1 + B_t|} \\
      &\leq \frac{1}{|1 + B_t|} \bigg(  \frac{p_\text{max}}{p_\text{min}} \frac{ \varepsilon}{\norm{\pi(\textbf{x}) - \textbf{x}} - \varepsilon } + |C_t| \bigg) \\
      &\xrightarrow[t \to 0]{} \frac{p_\text{max}}{p_\text{min}} \frac{ \varepsilon}{\norm{\pi(\textbf{x}) - \textbf{x}} - \varepsilon }  
\end{align*}
Since $\varepsilon$ can be chosen arbitrarily small this finishes the proof. \qed

\subsection*{Proof of Theorem \ref{thm:score_orthogonal}}
Beginning with $\textbf{n}$ and extending to an orthonormal basis ($\textbf{n}$, $\boldsymbol{\nu}_1$, ... , $\boldsymbol{\nu}_{d-1}$) of $\mathbb{R}^d$, we have
\begin{gather*}
    \text{S}_{\cos} (\textbf{n}, \nabla_\textbf{x} \ln p_t(\textbf{x})) = \frac{\textbf{n}^T\nabla_\textbf{x} \ln p_t(\textbf{x})}{\norm{\nabla_\textbf{x} \ln p_t(\textbf{x})}} 
    = \frac{\textbf{n}^T(\nabla_\textbf{x} \ln p_t(\textbf{x}))}{\sqrt{ (\textbf{n}^T\nabla_\textbf{x} \ln p_t(\textbf{x}))^2 + \sum_{i=1}^{d-1}(\boldsymbol{\nu}_i^T\nabla_\textbf{x} \ln p_t(\textbf{x}))^2 }} \\
    = \frac{1}{\sqrt{ 1 + \sum_{i=1}^{d-1}\big( \frac{\boldsymbol{\nu}_i^T\nabla_\textbf{x} \ln p_t(\textbf{x})}{\textbf{n}^T\nabla_\textbf{x} \ln p_t(\textbf{x})} \big)^2 }} 
     \xrightarrow[t \to 0]{} 1,
\end{gather*}
where in taking the limit we applied the Theorem \ref{thm:master_thm}. \qed

\subsection*{Proof of Corollary \ref{cor:score_ratio}}
Let $(\boldsymbol{\tau}_1, ..., \boldsymbol{\tau}_k)$ be an orthonormal basis of the tangent space $T_{\pi(\textbf{x})}\mathcal{M}$ and extend $\textbf{n}$ to an orthonormal basis $(\textbf{n}, \boldsymbol{\eta}_1, ..., \boldsymbol{\eta}_{d-k-1})$ of the normal space $\mathcal{N}_{\pi(\textbf{x})}\mathcal{M}$.
Then the projection of the score on the tangent space is given by $\mathbf{T} \nabla_\textbf{x} \ln p_t(\textbf{x}) = \sum_{i=1}^k (\boldsymbol{\tau}_i^T\nabla_\textbf{x} \ln p_t(\textbf{x})) \boldsymbol{\tau}_i$, while the projection on the normal space is  $\mathbf{N} \nabla_\textbf{x} \ln p_t(\textbf{x}) = (\textbf{n}^T\nabla_\textbf{x} \ln p_t(\textbf{x})) \textbf{n} +  \sum_{i=1}^{d-k-1} (\boldsymbol{\eta}_i^T\nabla_\textbf{x} \ln p_t(\textbf{x})) \boldsymbol{\eta}_i$. Therefore
\begin{gather*}
    \frac{\norm{\mathbf{T} \nabla_\textbf{x} \ln p_t(\textbf{x}) }}{\norm{\mathbf{N} \nabla_\textbf{x} \ln p_t(\textbf{x}) }} = \sqrt{\frac{\sum_{i=1}^k (\boldsymbol{\tau}_i^T\nabla_\textbf{x} \ln p_t(\textbf{x}))^2 }{ (\textbf{n}^T\nabla_\textbf{x} \ln p_t(\textbf{x}))^2 +  \sum_{i=1}^{d-k-1} (\boldsymbol{\eta}_i^T\nabla_\textbf{x} \ln p_t(\textbf{x}))^2 }}
    \\ = \sqrt{\frac{\sum_{i=1}^k \big( \frac{\boldsymbol{\tau}_i^T\nabla_\textbf{x} \ln p_t(\textbf{x})}{\textbf{n}^T\nabla_\textbf{x} \ln p_t(\textbf{x})} \big)^2 }{ 1 +  \sum_{i=1}^{d-k-1} \big( \frac{\boldsymbol{\eta}_i^T\nabla_\textbf{x} \ln p_t(\textbf{x})}{\textbf{n}^T\nabla_\textbf{x} \ln p_t(\textbf{x})} \big)^2 }}
    \xrightarrow[t \to 0]{} 0
\end{gather*}
where in taking the limit we applied the Theorem \ref{thm:master_thm}. \qed

\newpage

\section{Details on the design of synthetic image manifolds}\label{Appendix: design of synthetic image manifolds}

\textbf{Squares image manifold:} We generated the \textit{k-squares dataset} whose intrinsic dimension is controllable and is set to $k$. This dataset comprises 32 × 32 pixel images of squares, generated by first establishing fixed square center locations and side lengths (either 3 or 5 units) across all datapoints. For each square in a given image, we uniformly sampled a brightness value from the unit interval for all pixels within the square's boundary, summing values at points of intersection. The dimension of this manifold is equal to the number of squares $k$ and the ambient space dimension is $32\times32=1024$. We provide samples in Figure \ref{fig:squares}.

\textbf{Gaussian blobs image manifold:} The Squares image manifold is contained in a low dimensional linear subspace which allowed PPCA to estimate the dimension very well. For this reason, we constructed a synthetic image manifold of known dimension, which cannot be contained in any low dimensional linear subspace.  We replace the squares by Gaussian blobs (i.e. brightness of pixels withing each blob is proportional to a Gaussian pdf). We randomly pick the centers of the Gaussian blobs which remain fixed for all datapoints. For each datapoint, we sample the standard deviation of each Gaussian blob uniformly from the interval $[1,5]$. The dimension of this manifold is equal to the number of Gaussian blobs. We provide samples in Figure \ref{fig:gaussian-blobs}.

\begin{figure}[h]
     \centering
     \begin{minipage}{.3\textwidth}
         \centering
         \includegraphics[width=.75\columnwidth]{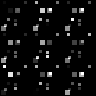}
     \end{minipage}
     \begin{minipage}{.3\textwidth}
         \centering
         \includegraphics[width=.75\columnwidth]{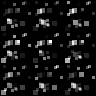}
     \end{minipage}
     \begin{minipage}{.3\textwidth}
         \centering
         \includegraphics[width=.75\columnwidth]{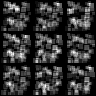}
     \end{minipage} 
     \caption{Nine samples from Squares image manifolds of dimensions 10, 20 and 100 (from left to right).}
     \label{fig:squares}
     \vspace{10pt}
    \begin{minipage}{.3\textwidth}
         \centering
         \includegraphics[width=.75\columnwidth]{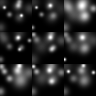}
     \end{minipage}
     \begin{minipage}{.3\textwidth}
         \centering
         \includegraphics[width=.75\columnwidth]{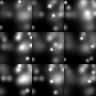}
     \end{minipage}
     \begin{minipage}{.3\textwidth}
         \centering
         \includegraphics[width=.75\columnwidth]{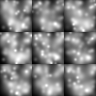}
     \end{minipage} 
     \caption{Nine samples from Gaussian blob image manifolds of dimensions 10, 20 and 100 (from left to right).}
     \label{fig:gaussian-blobs}     
\end{figure}

\newpage

\section{Additional Experimental Results}
\label{appendix:additional_experimental_results}

\subsection{Euclidean Data}
\label{appendix:additional_euclidean}

\begin{figure}[H]
\begin{minipage}[t]{.45\textwidth}
    \centering
    \includegraphics[width=.99\textwidth]{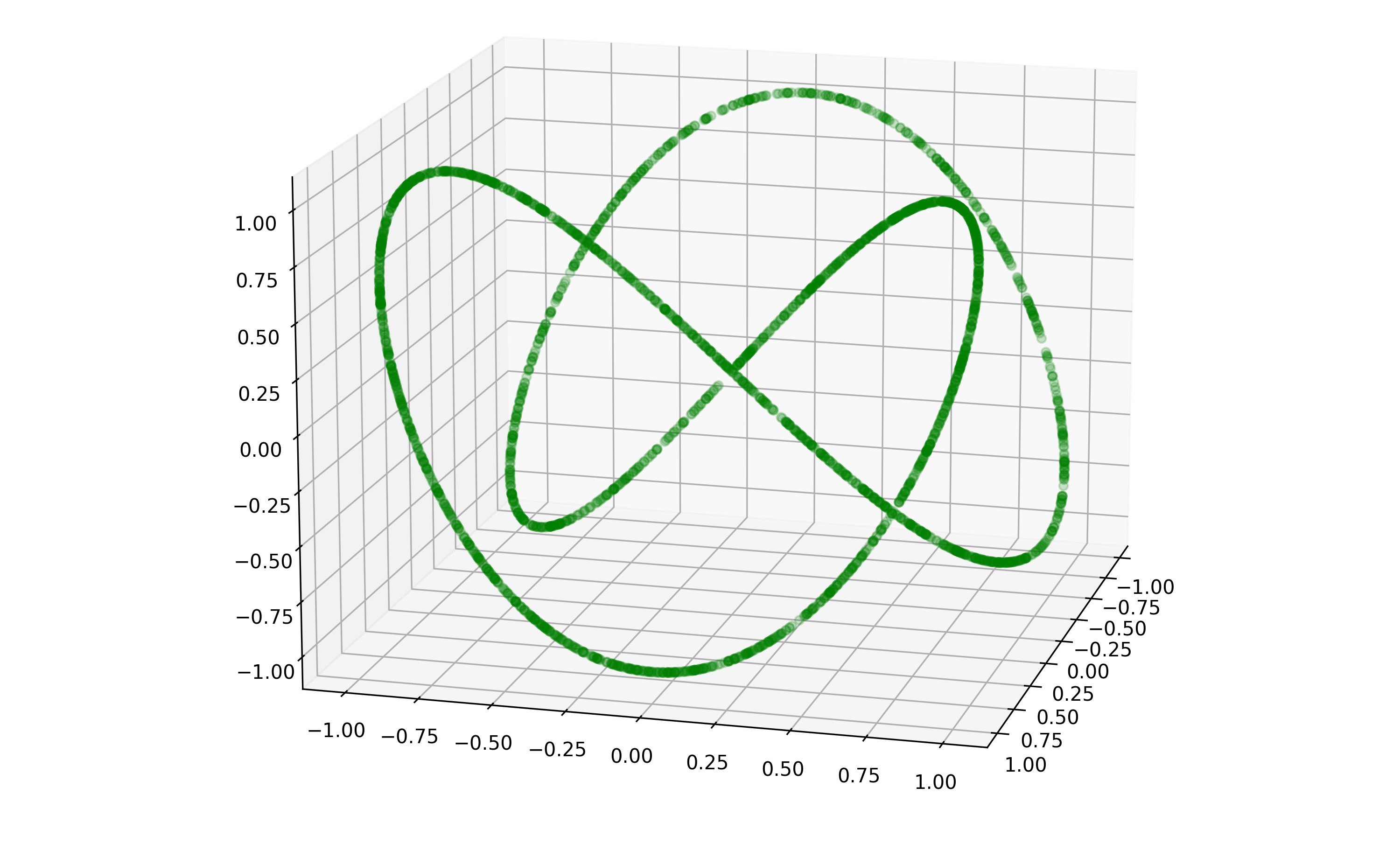}
    \caption{Projection of the spaghetti line on the first three dimensions.}
    \label{fig:line_samples}
\end{minipage}
\hspace{5mm}
\begin{minipage}[t]{.45\textwidth}
    \centering
    \includegraphics[width=.99\textwidth]{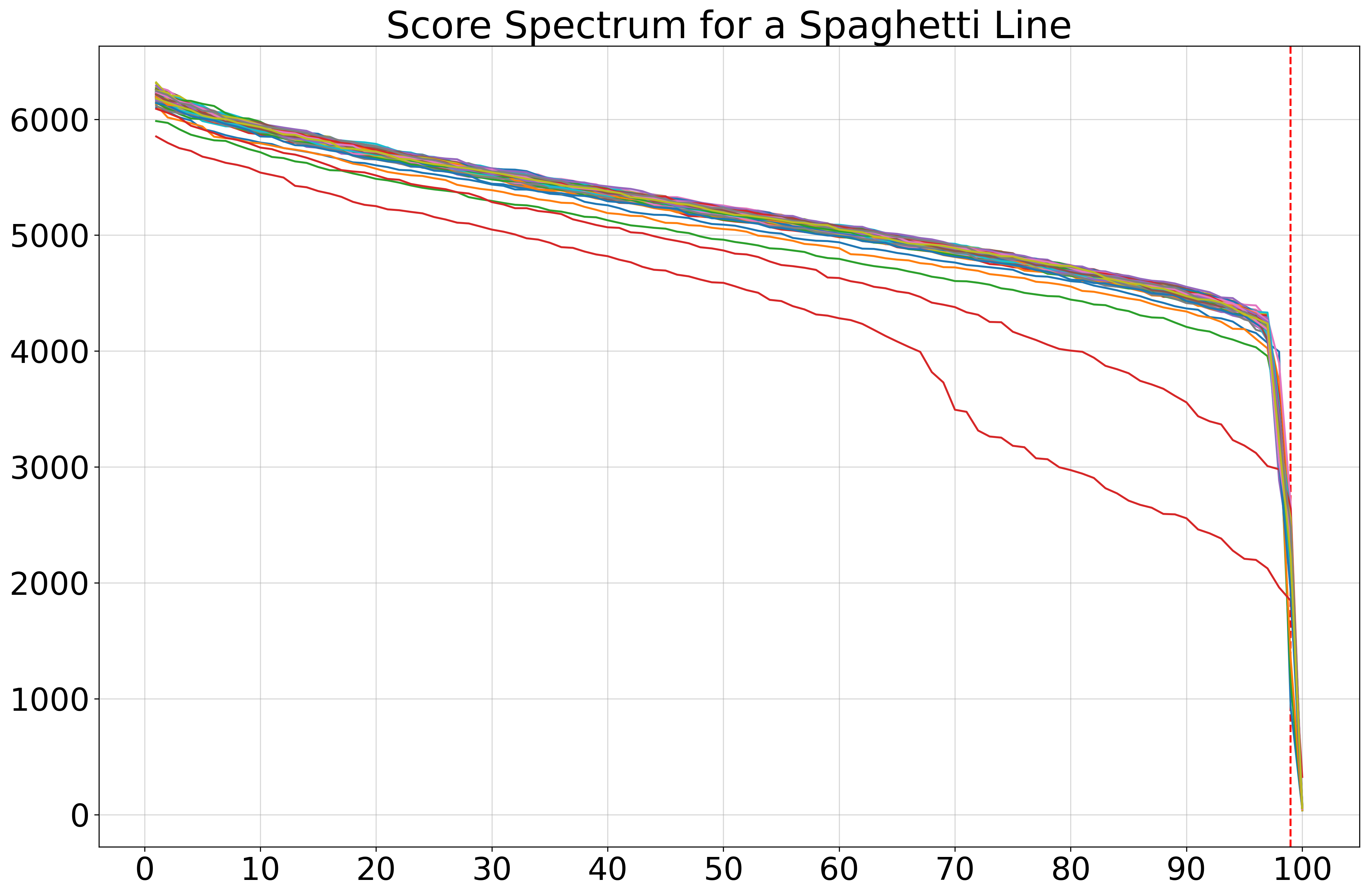}
    \caption{Score spectrum of the spaghetti line. The last singular value clearly vanishes indicating that the intrinsic dimensionality of the manifold is equal to one.}
    \label{fig:line}
\end{minipage}
\end{figure}

\begin{figure}[H]
\begin{minipage}[t]{.45\textwidth}
    \centering
    \includegraphics[width=\linewidth]{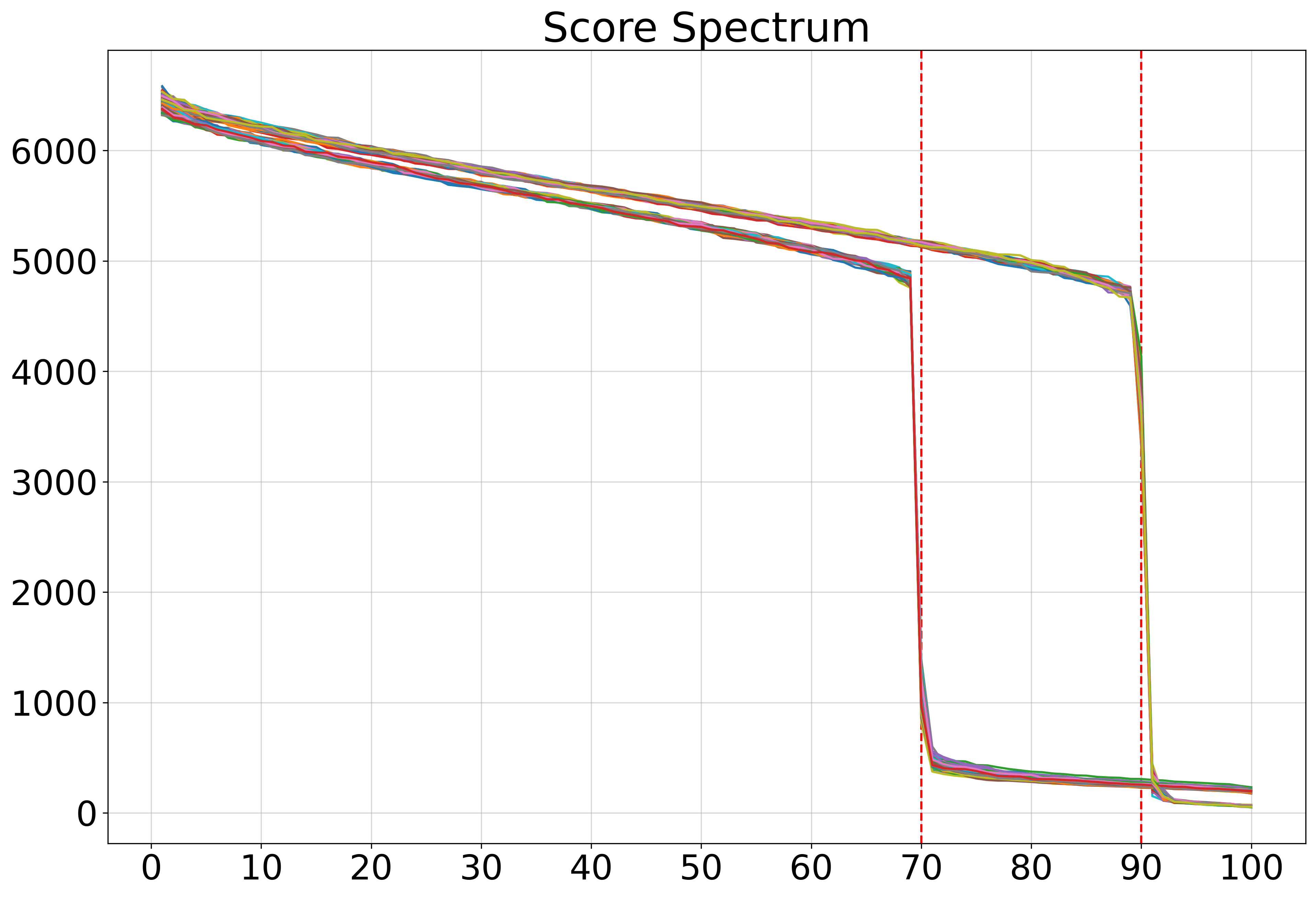}
    \caption{Score spectrum for the union of $k$-spheres ($k_1=10, k_2=30$). The separated drops in the spectra clearly show that the data comes form the union of two manifolds of different dimensions.}
    \label{fig:unions_spectrum}
\end{minipage}
\hspace{5mm}
\begin{minipage}[t]{.45\textwidth}
   \centering
    \includegraphics[width=0.99\textwidth]{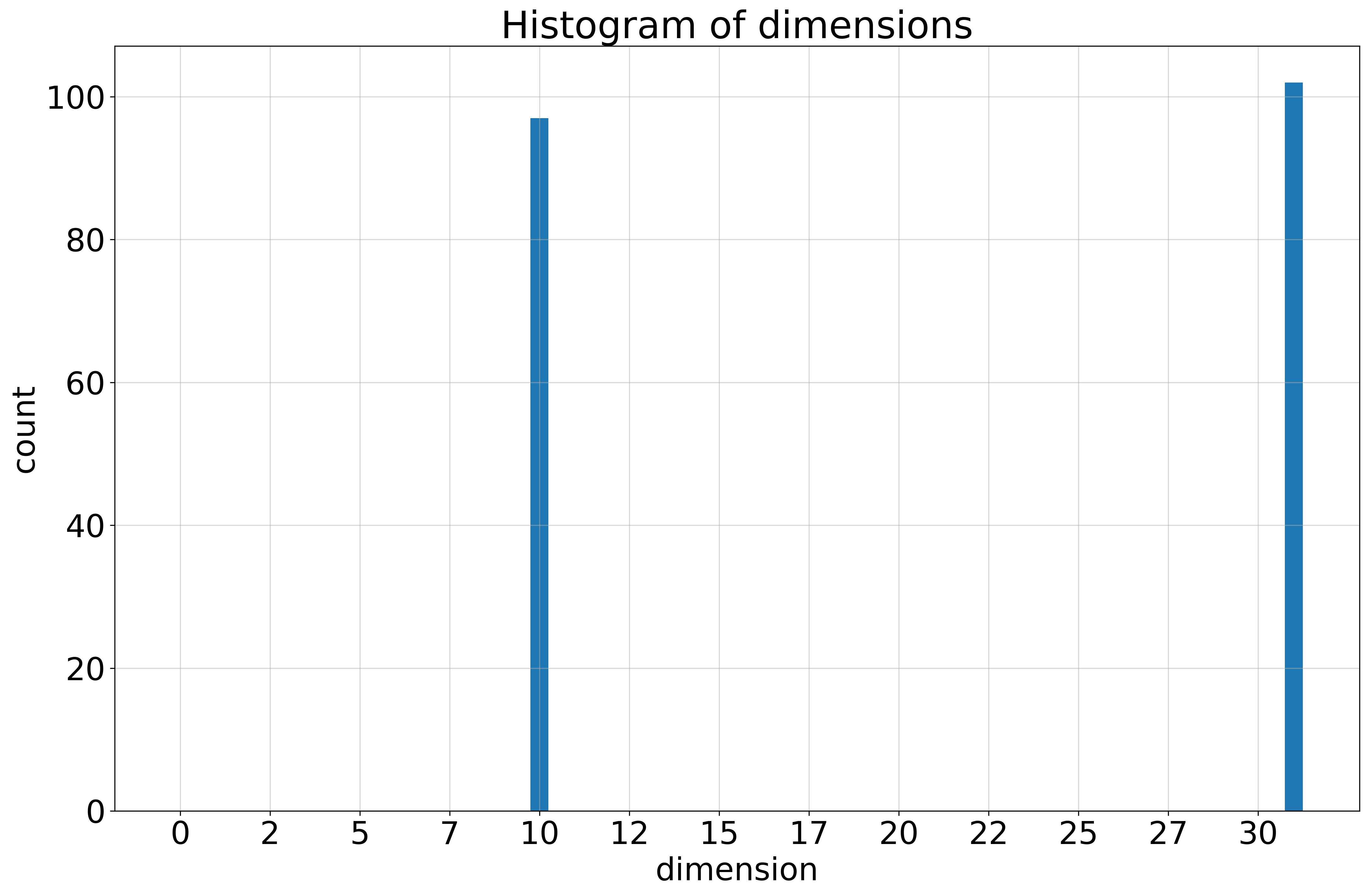}
    \caption{The histogram of estimated dimensions for the union of $k$-spheres ($k_1=10, k_2=30$). The counts are taken over estimates $\hat{k}(\textbf{x}_0^{(i)})$ at different points $\textbf{x}_0^{(i)}$.}
    \label{fig:union_dims}
\end{minipage}
\end{figure}

\subsection{Synthetic Image Data}
\label{appendix:additional_image}

\begin{figure}[H]
\centering
\begin{minipage}[t]{.45\textwidth}
    \centering
    \includegraphics[width=.95\textwidth]{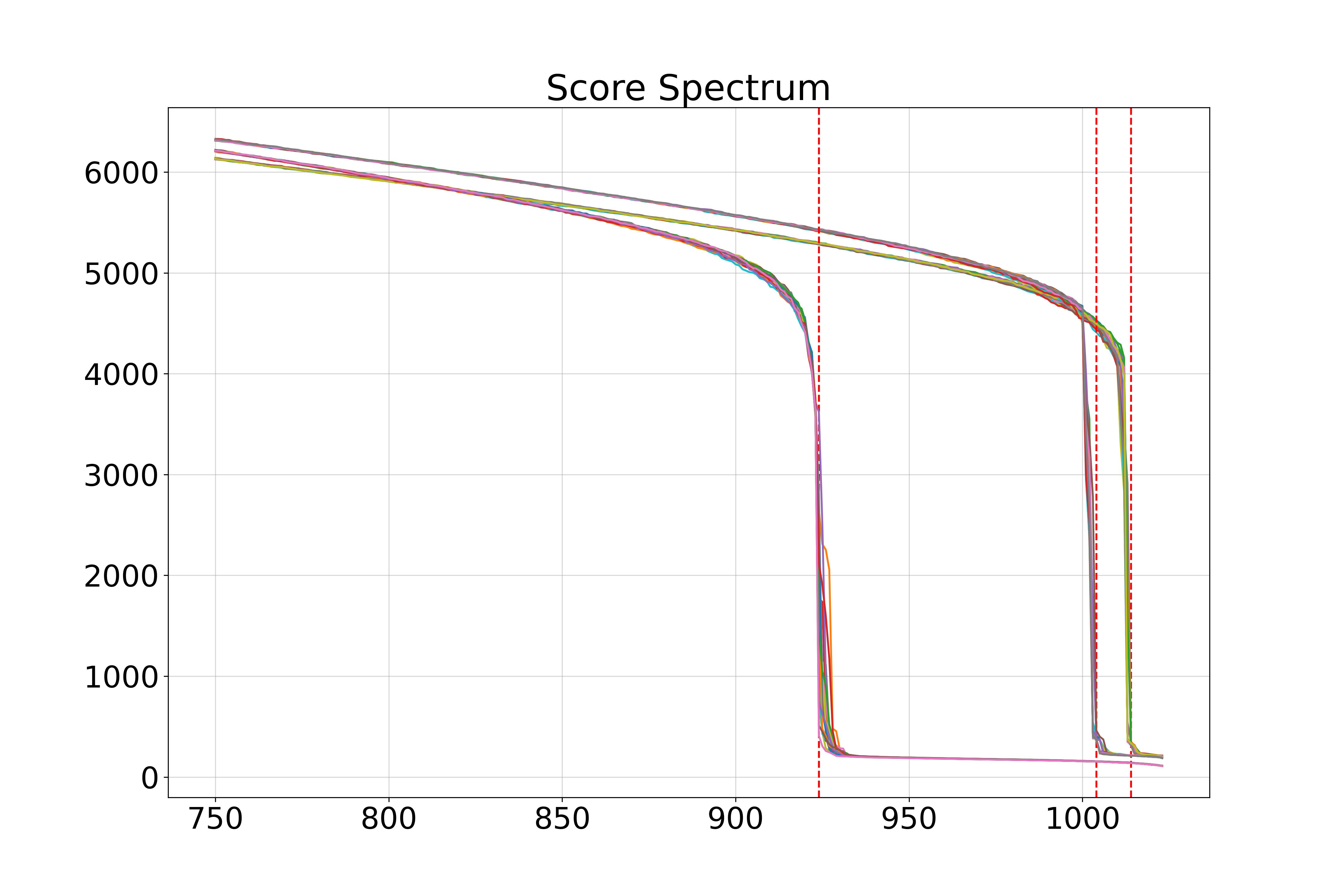}\\
    \includegraphics[width=.95\textwidth]{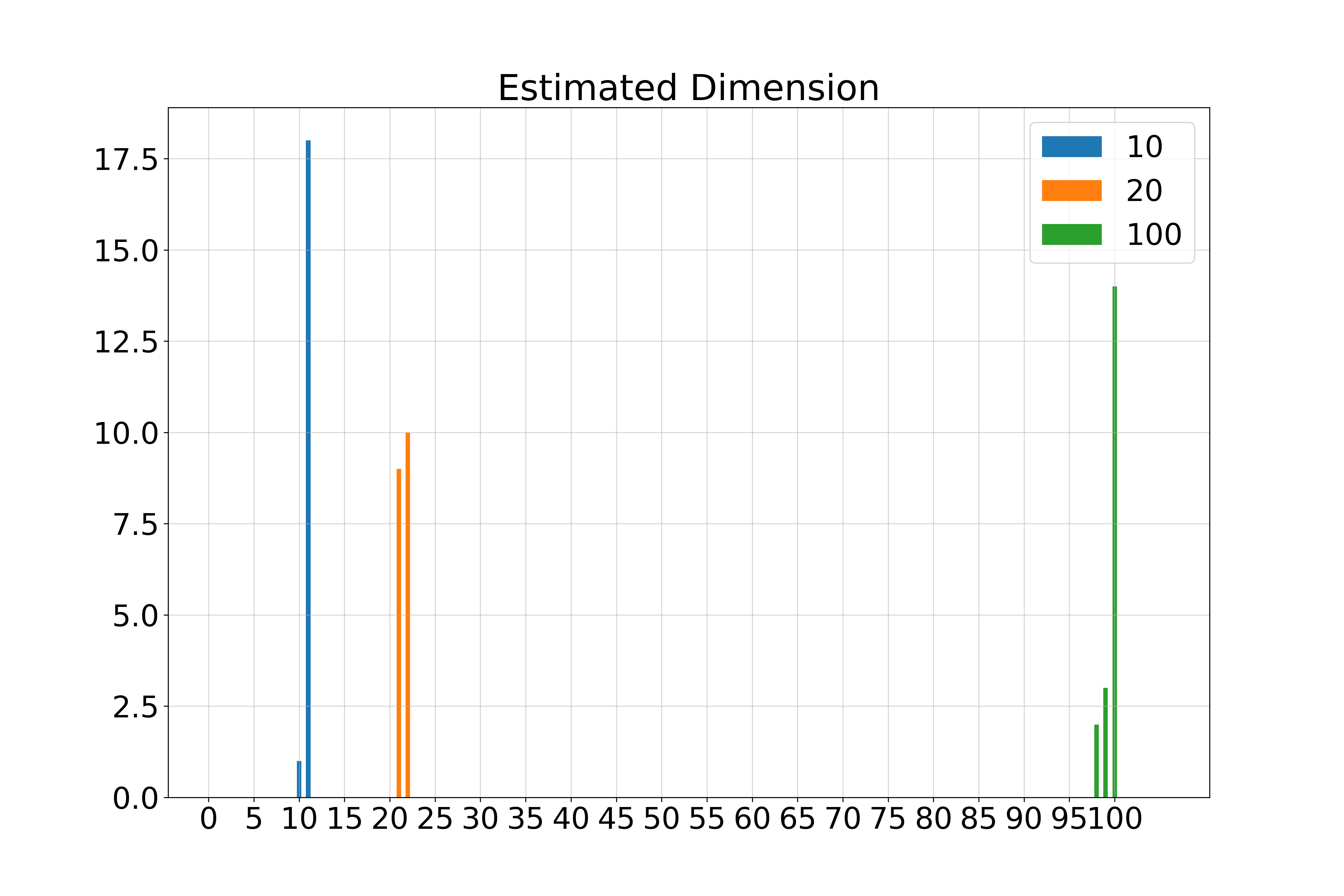}
    \caption{Score spectra and histogram of estimated dimension based on the score spectrum of the Squares image manifold of dimensions 10, 20 and 100.}
    \label{fig:squares_spectrum}
\end{minipage}
\hspace{10mm}
\begin{minipage}[t]{.45\textwidth}
    \centering
    \includegraphics[width=.95\textwidth]{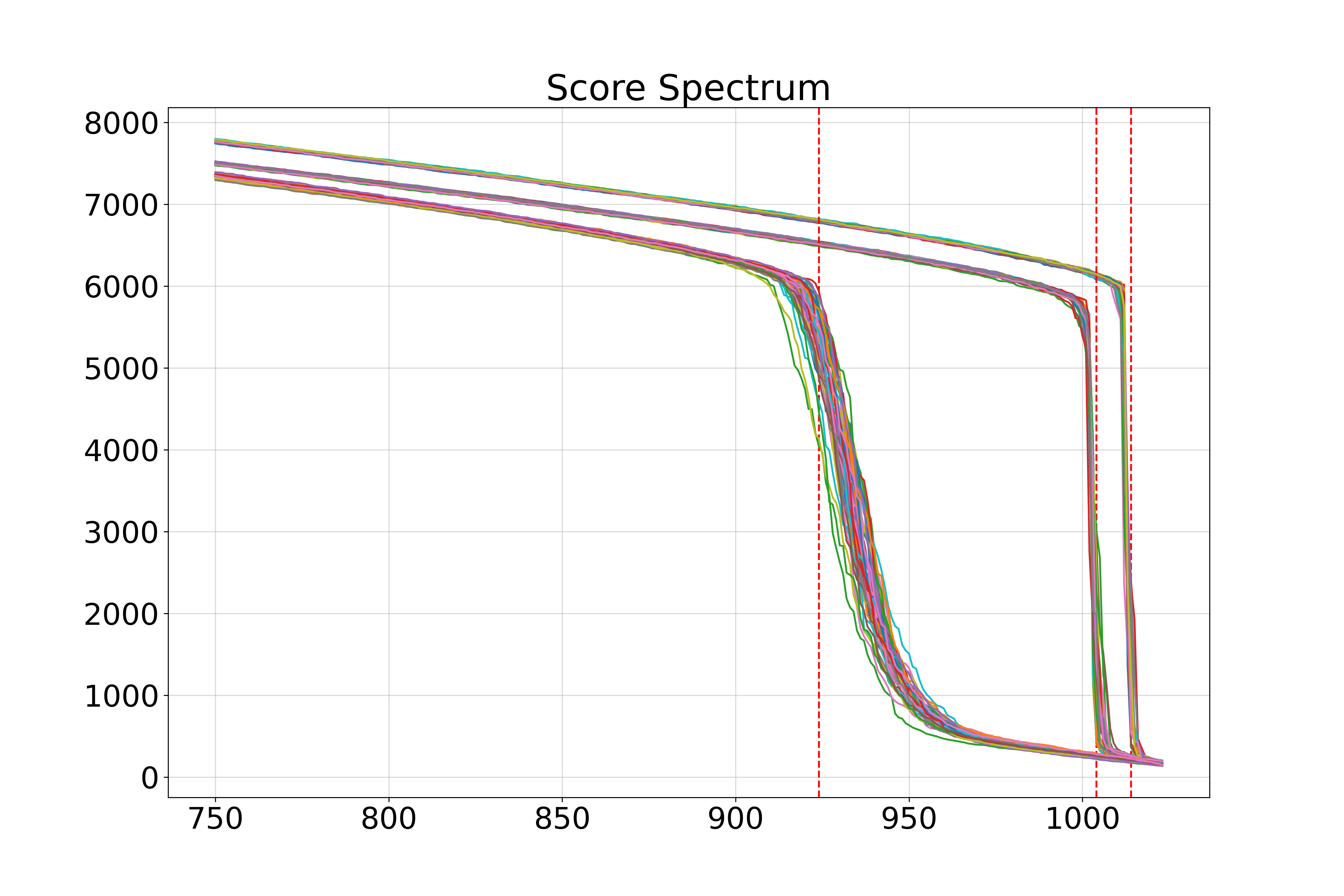}\\
    \includegraphics[width=.95\textwidth]{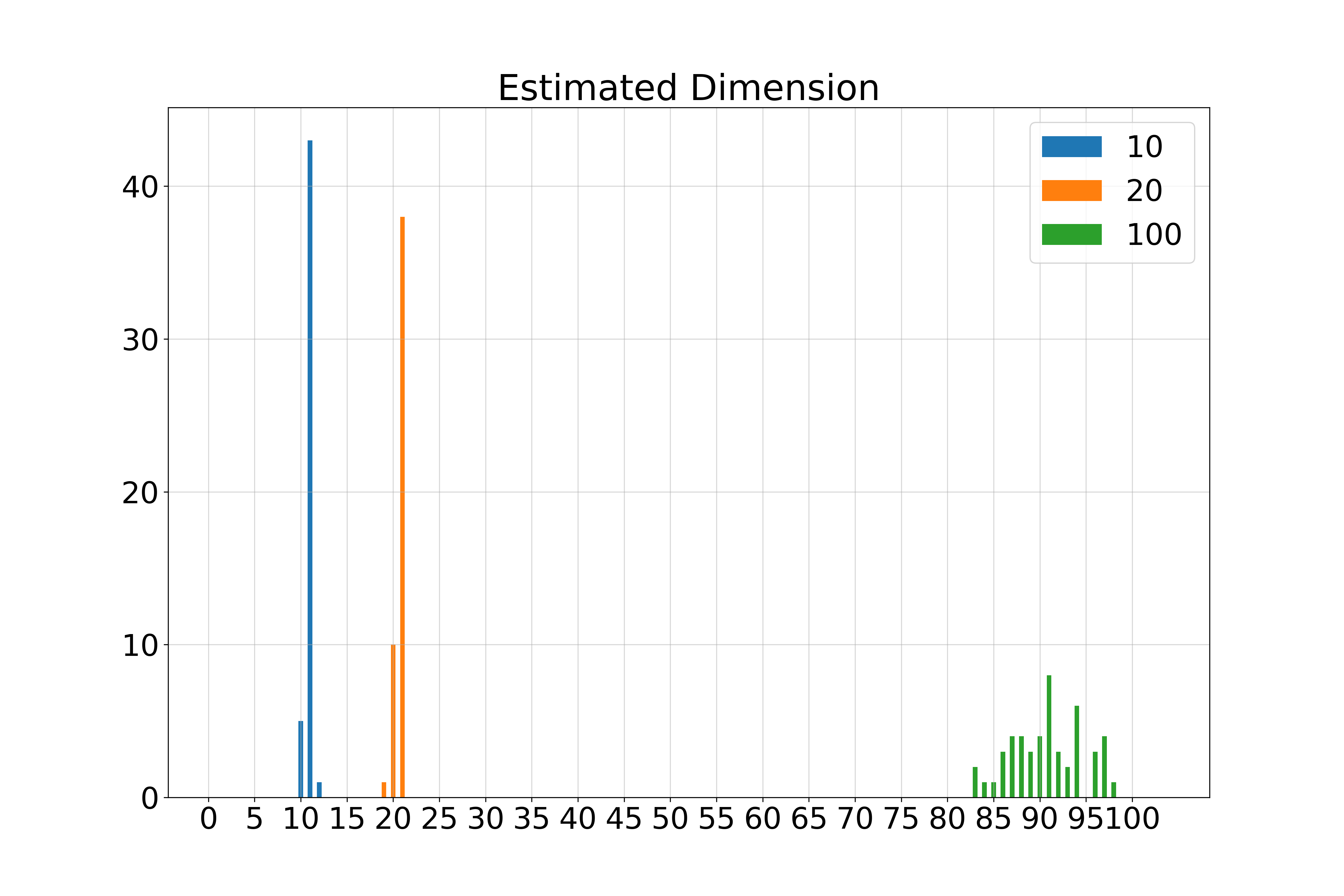}
    \caption{Score spectra and histogram of estimated dimension based on the score spectrum of the Gaussian blobs image manifold of dimensions 10, 20 and 100.}
    \label{fig:gaussians_spectrum}
\end{minipage}
\end{figure}

\subsection{MNIST}
\label{sec:Additional Experimental Results for MNIST}
\begin{figure}[h!]
    \centering
    \includegraphics[width=0.99\textwidth]{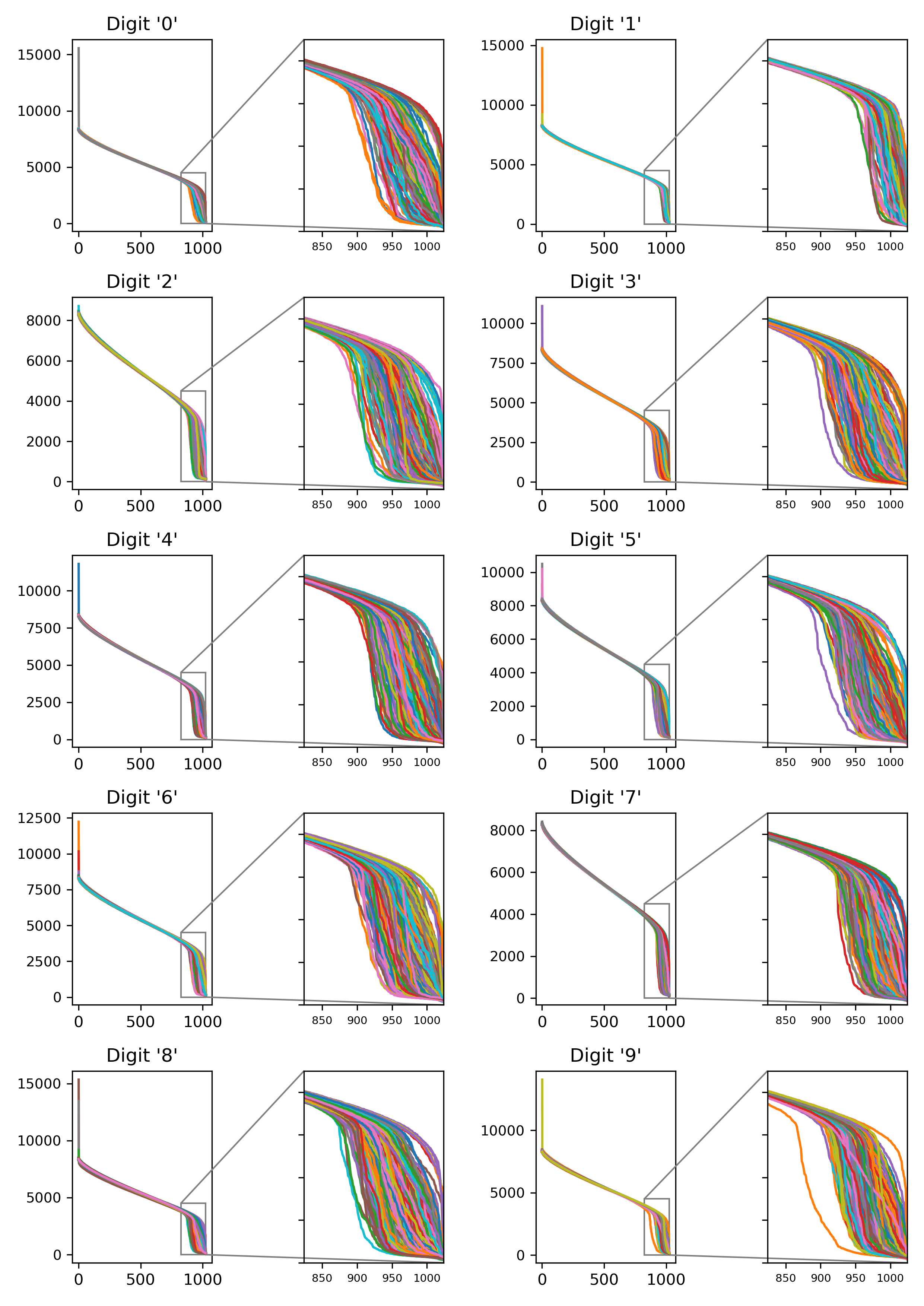}
    \caption{MNIST score spectra for all digits}
    \label{fig:score_spectra_mnist_extended}
\end{figure}

In Figure \ref{fig:score_spectra_mnist_extended} we present $500$ score spectra evaluated at $500$ different instances for each digit. We observe that a considerable number of spectra indicate lower manifold dimension than the dimension documented in Table \ref{tab:estimated_mnist_dimensions}. This deviation can be attributed to amplified
geometrical and statistical error at the respective evaluation points. 

We choose the maximum estimated dimension as our conjecture of the intrinsic dimension under the premise that a collapse of the spectrum at a smaller dimension than the dimension of the normal space is extremely unlikely from a probabilistic standpoint.  However, it is feasible to encounter a spectrum collapse suggesting a higher normal space dimension, hence a lower intrinsic dimension, due to the intensified geometric and approximation error at the point of evaluation. It is noteworthy that our estimation strategy locates the drop at the position of the maximal gradient, a practice that may marginally inflate the intrinsic dimension estimate, as exemplified in the Squares and Gaussian blobs manifolds. Therefore, if our reported dimension for each MNIST digit is not precise, it is either a minor overestimation or a lower limit of the true dimension.

\newpage
\section{Robustness analysis}
\label{appendix:robust}
\subsection{Robustness to score approximation error}
As we discussed in the previous sections, our method is guaranteed to work given a perfect score approximation for sufficiently small $t$. However, in practice there will be an approximation error resulting from finite training data, limited model capacity and imperfect optimization. Therefore, we conduct an empirical analysis of the influence of the error in score approximation on the produced estimate of the dimension. We train a model $s_\theta$ on a uniform distribution on $25$-sphere and then we corrupt the output of the model with a Gaussian perturbation $e \sim \mathcal{N}(0, \sigma^2_e\textbf{I})$. Then, we produce score spectra by applying our method to the resulting corrupted scores. We repeat this experiment for different intensities of noise $\sigma_e$. We pick $\sigma_e$ so that the noise norm to score norm ratio $r = \mathbb{E}[\norm{e}] / \mathbb{E}_{x_{t_0} \sim p_{t_0}(\textbf{x}_{t_0} | \textbf{x}_0)}[\norm{s(\textbf{x}_{t_0} , t_0)}]$ has a prescribed value. We find that as we increase the intensity of noise singular values corresponding to the tangential component start to increase causing the gap in the score spectrum  to diminish. This is expected since the noise added to the score vectors has a tangential component.  However, for values of $r < 0.5$ our method is still producing a visible drop in the spectrum at the correct point. The results are presented in Figure \ref{fig:robustness_score}.

\begin{figure}[h!]
    \centering
    \includegraphics[width=\textwidth]{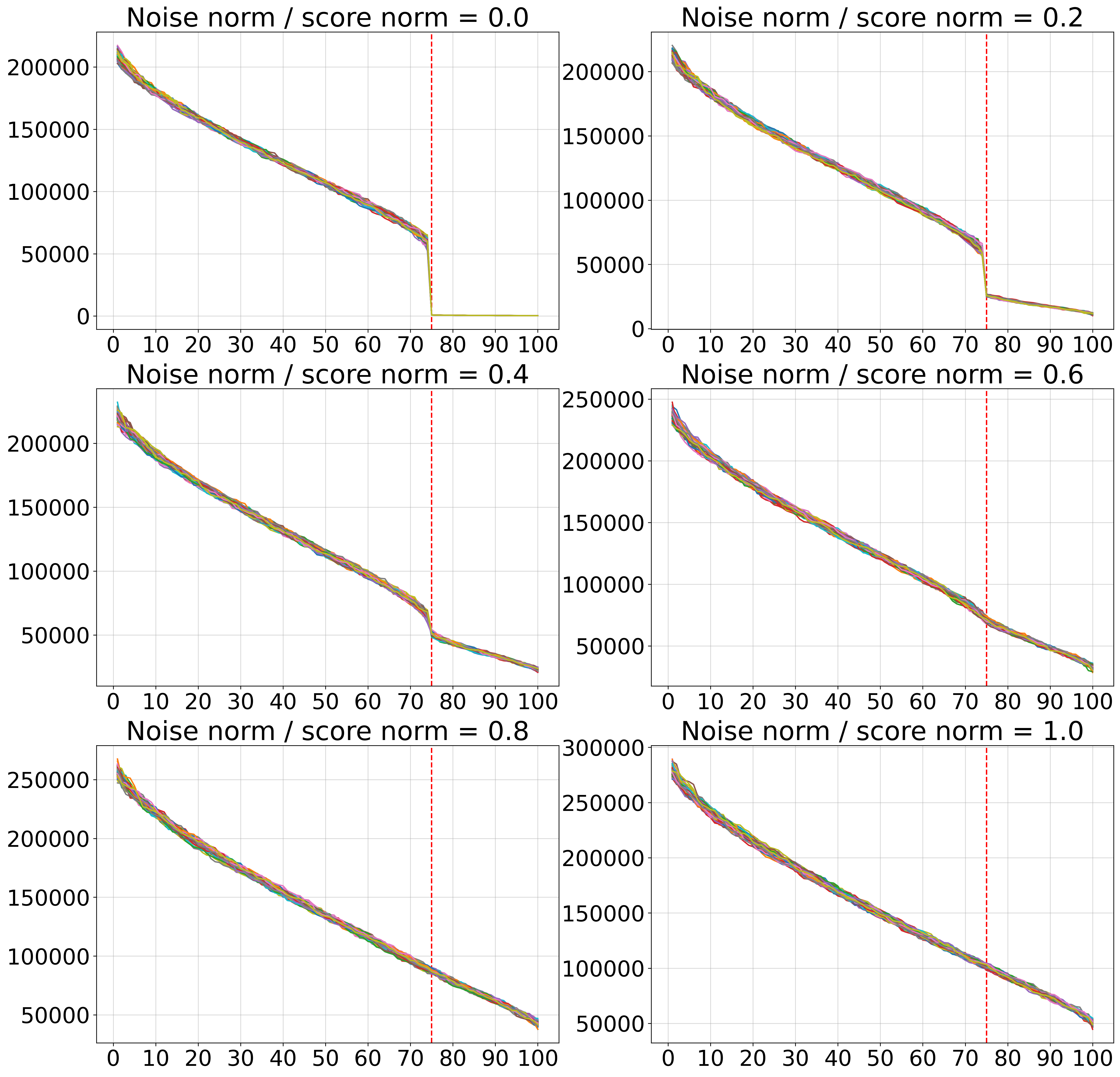}
    \caption{Score spectra for noise corrupted score model on 25-sphere.}
    \label{fig:robustness_score}
\end{figure}

\newpage
\subsection{Robustness to non-uniform distribution on the manifold}

\begin{wraptable}{L}{0.6\textwidth}
    \begin{center}
    \small
    \begin{tabular}{ccccc}
    \toprule
    & Uniform & $\alpha=1$ & $\alpha=0.75$ & $\alpha=0.5$  \\
    \midrule
    \multirow{1}{*}{Ground Truth} 
    &10 & 10 & 10  & 10    \\
    \midrule
    \multirow{1}{*}{Ours}
    &11 & 10 & 10  & 7 \\
    \midrule
    
    \multirow{1}{*}{MLE (m=5)}
    &9.61 & 5.37 & 4.83  & 4.12   \\
    \midrule

    \multirow{1}{*}{MLE (m=20)}
    &9.46 & 4.99 & 4.49  & 3.82 \\
    \midrule

    \multirow{1}{*}{Local PCA}
    & 11 & 5 & 4  & 3  \\
    \midrule

    \multirow{1}{*}{PPCA}
    & 11 & 11 & 11  & 11  \\
    
    \bottomrule
    \end{tabular}
    \end{center}
    \caption{Dimensionality detection for non-uniform distribution. For our method the maximum over pointwise estimates $\hat{k}(\textbf{x}_0)$ is considered.}
    \label{tbl:non_uniform}
\end{wraptable}

We examine the robustness to our method to non-unifomity in data distribution on the manifold surface. Under perfect score approximation and sufficiently small $t_0$ our method is guaranteed to work, but we conduct an empirical study to investigate the behaviour in the presence of score approximation error and $t_0 > 0$ used in practice in diffusion models. We consider a $k$-sphere and sample the surface of the sphere in a non-uniform fashion. We obtain points on the $k$-sphere by sampling vectors $\pmb \theta$ of $k-1$ angles (in radians) from a Gaussian distribution $\mathcal{N}(0, \alpha \mathbf{I})$, where $\alpha \in \mathbb{R}$ is a constant that determines the degree of non-uniformity. Then, we embed the resulting points via a random isometry in a 100 dimensional ambient space. We sample $n=1000$ points $\textbf{x}_0^{(j)}$ from the manifold and at each point we estimate the dimensionality $\hat{k}(\textbf{x}_0^{(j)})$ via the score spectrum. The pointwise estimates are presented in Figure \ref{fig:non_uniform} and final estimates are shown in Table \ref{tbl:non_uniform}. For values of $\alpha \in \{1, 0.75\}$, we can still obtain an accurate estimate of the dimension if we take $\hat{k} = \max_{j=1,...,1000} \hat{k}(\textbf{x}_0^{(j)})$ the maximum over point-wise estimates. For an extremely concentrated distribution with $\alpha = 0.5$ the method underestimates the dimension, which indicates that the tangential component of the score was not approximately constant in the neighborhoods used to sample the scores. This problem could be further mitigated by approximating the score closer to the manifold and using smaller sampling neighborhoods (i.e. for smaller $t_0$) or sampling more points $\textbf{x}_0^{(j)}$. Notice that taking the maximum over $\hat{k}(\textbf{x}_0^{(j)})$ is theoretically justified (as long as we assume we are dealing with a connected manifold) since our method can underestimate due to geometric or approximation error (cf. sections \ref{sec:theory} and \ref{sec:limitations}, but it is unlikely to significantly overestimate.

\begin{figure*}
     \begin{minipage}[t]{0.5\textwidth}
         \includegraphics[width=\linewidth]{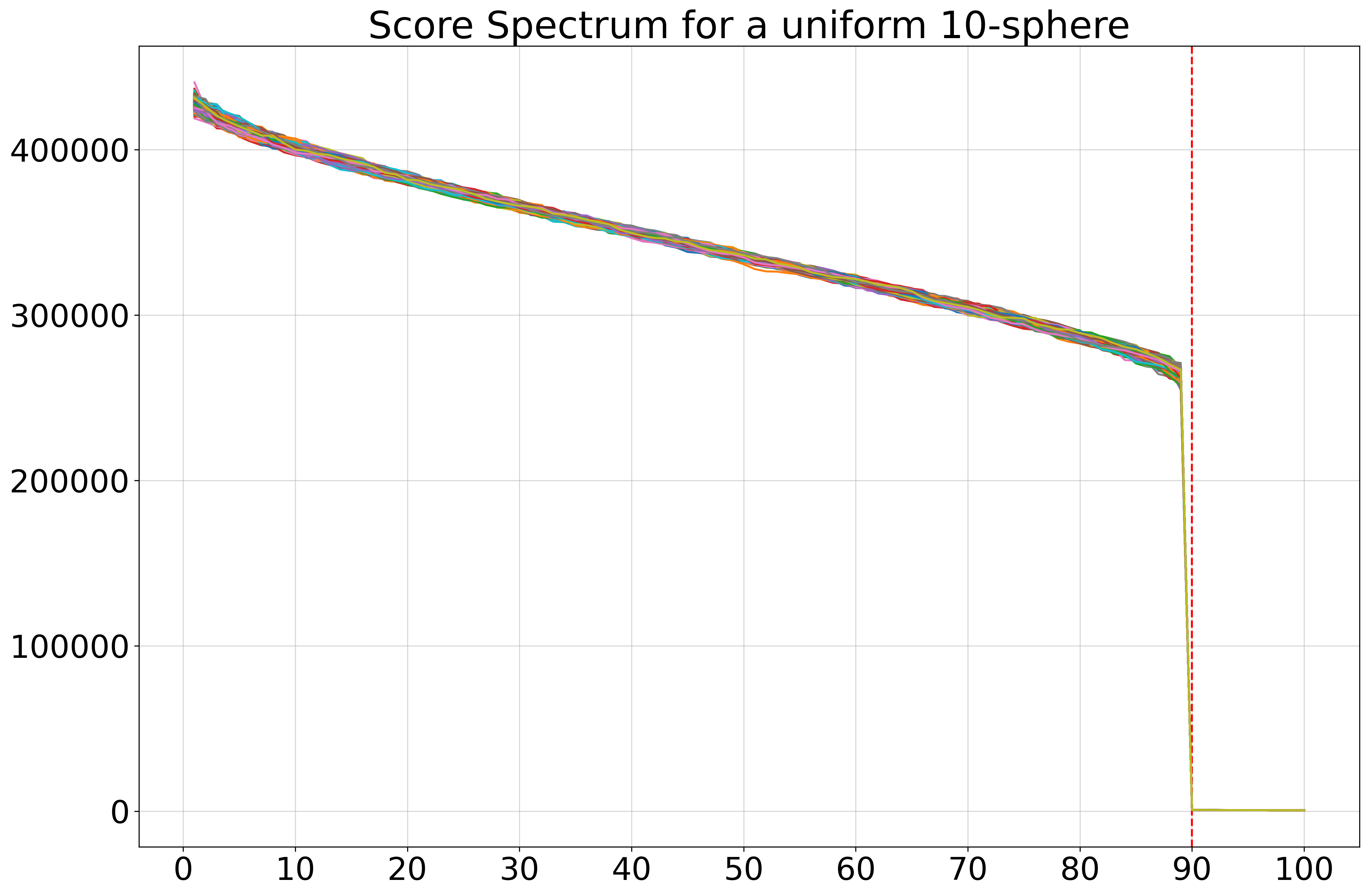}
         \includegraphics[width=\linewidth]{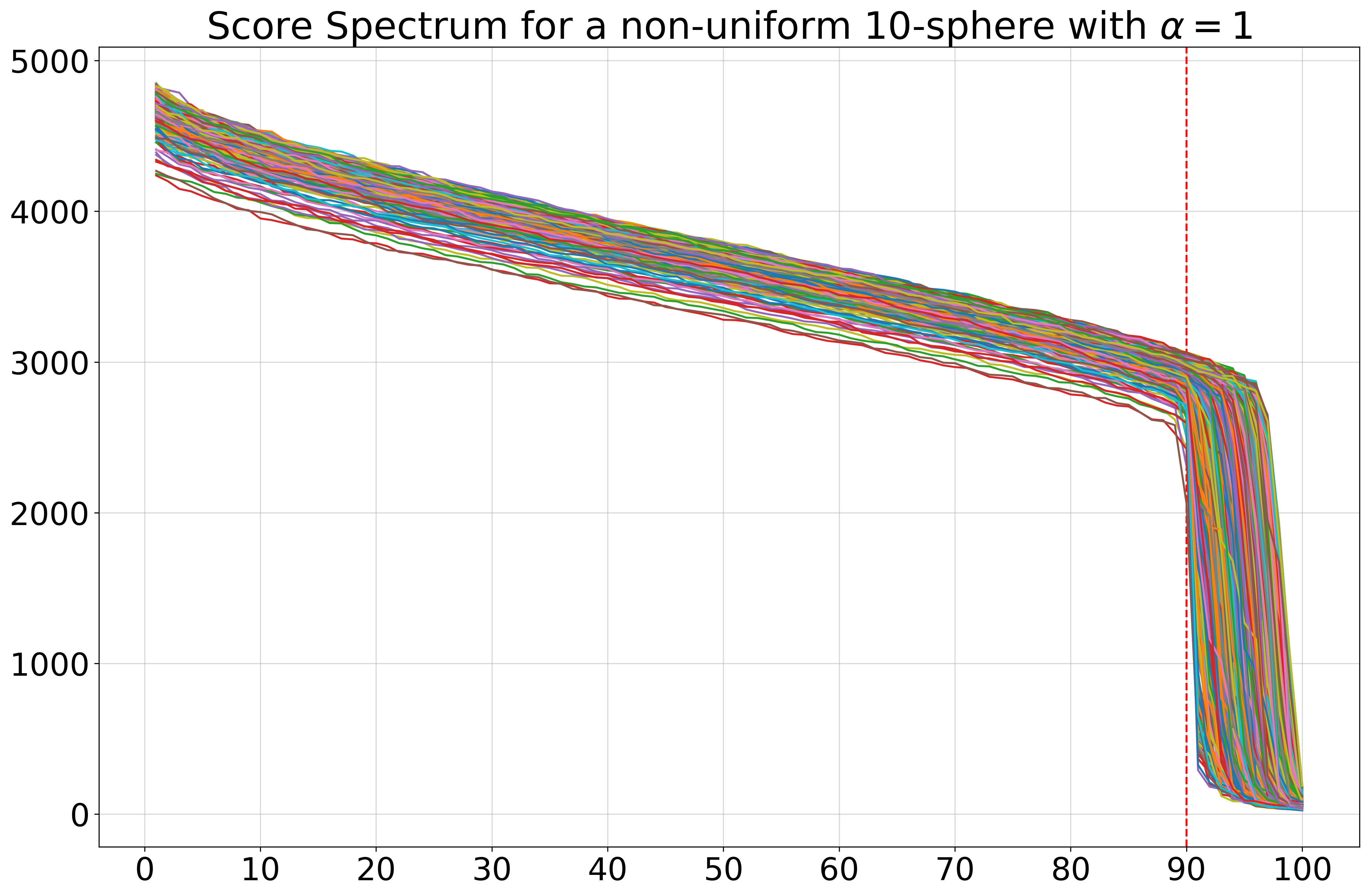}
         \includegraphics[width=\linewidth]{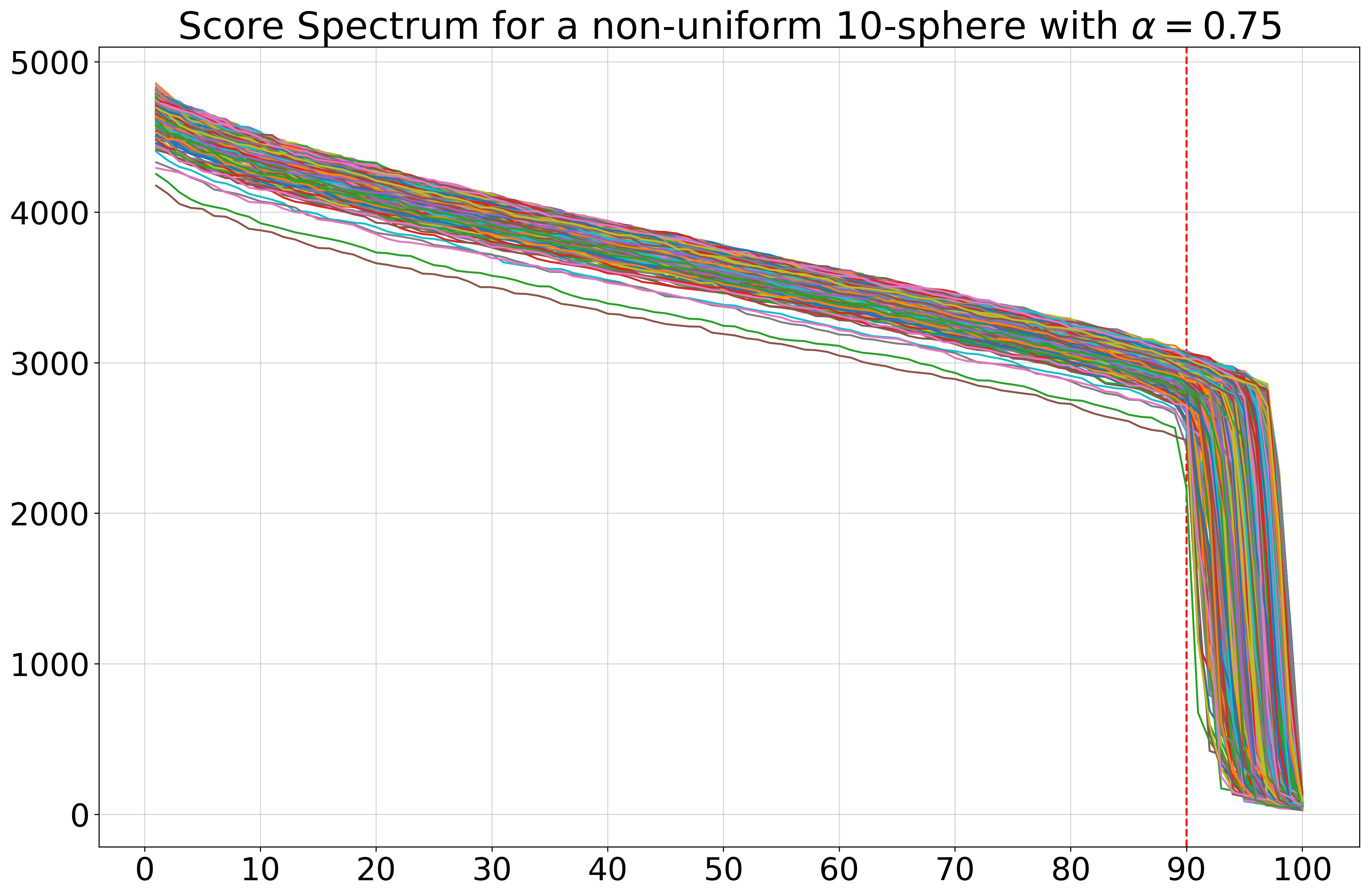}
         \includegraphics[width=\linewidth]{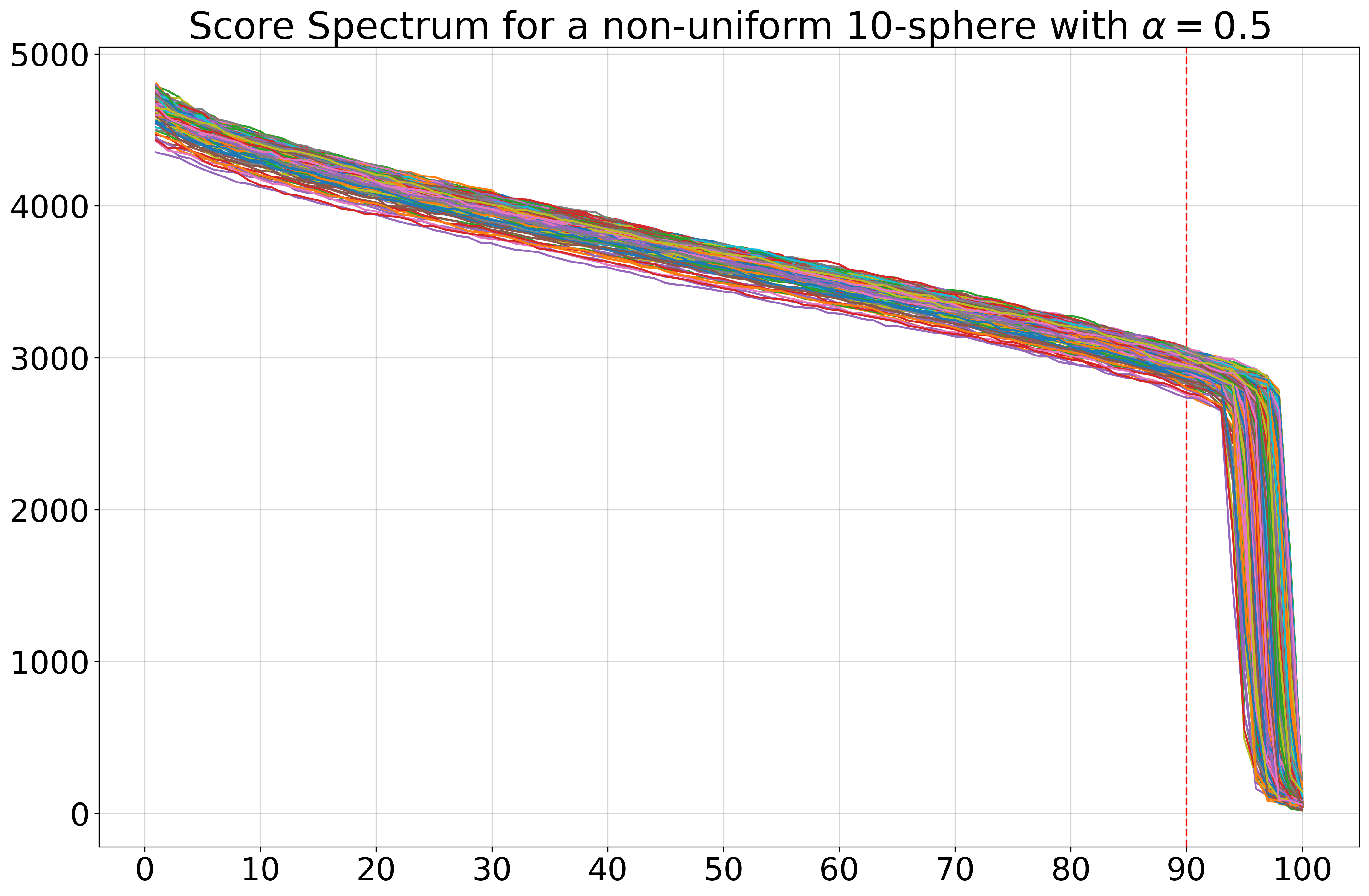}
     \end{minipage}
     \begin{minipage}[t]{0.5\textwidth}
         \includegraphics[width=\linewidth]{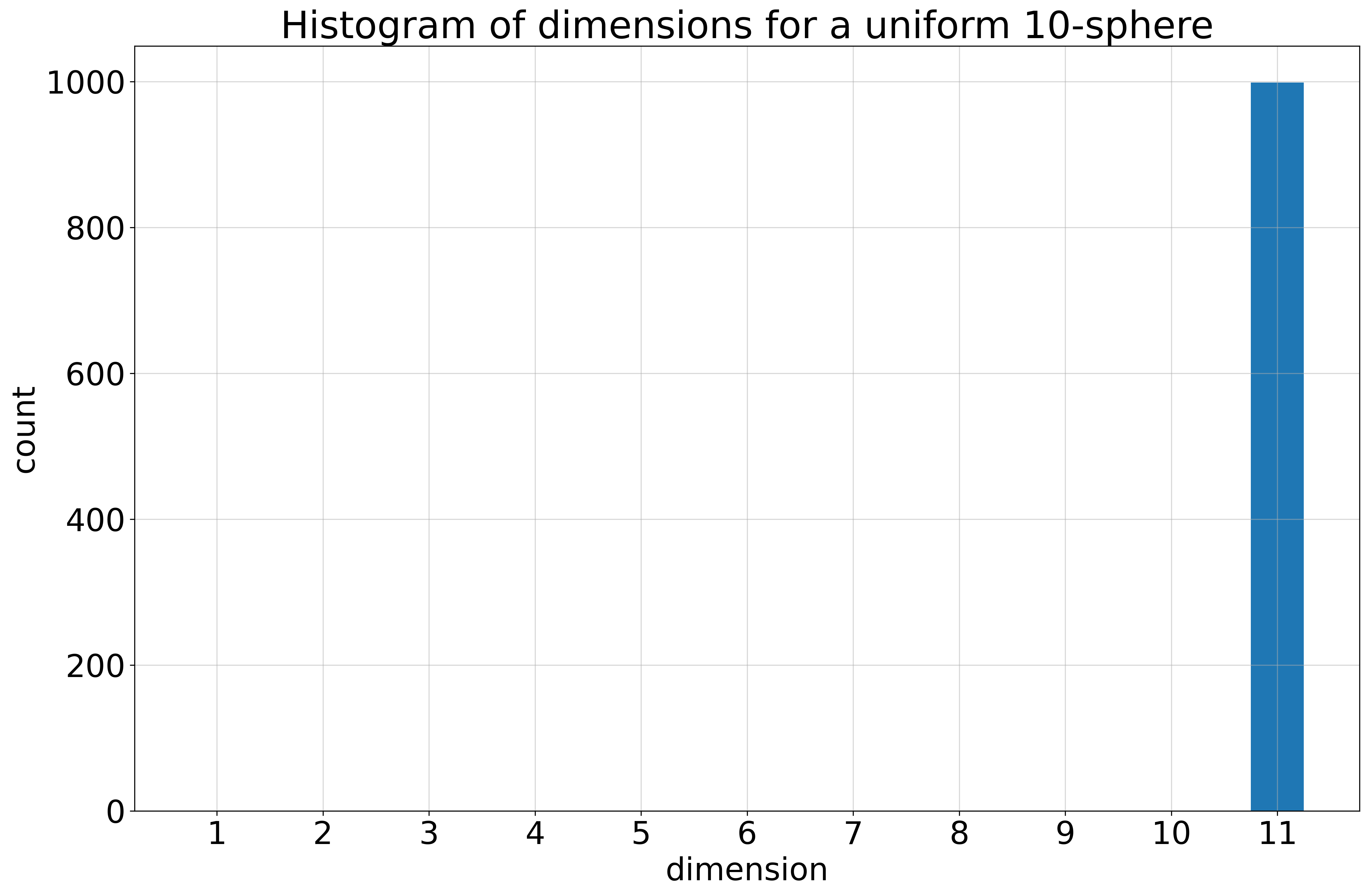}
         \includegraphics[width=\linewidth]{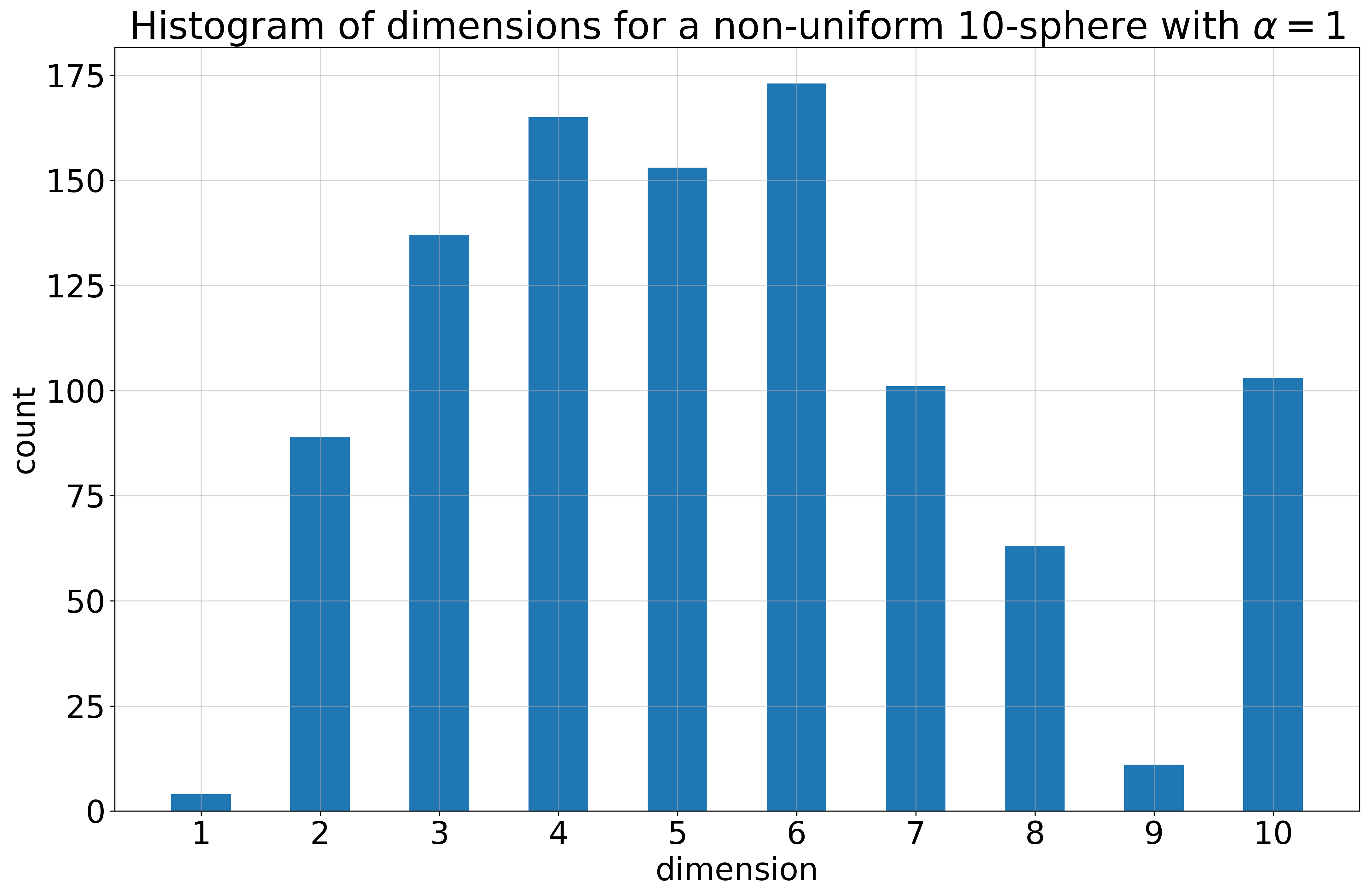}
         \includegraphics[width=\linewidth]{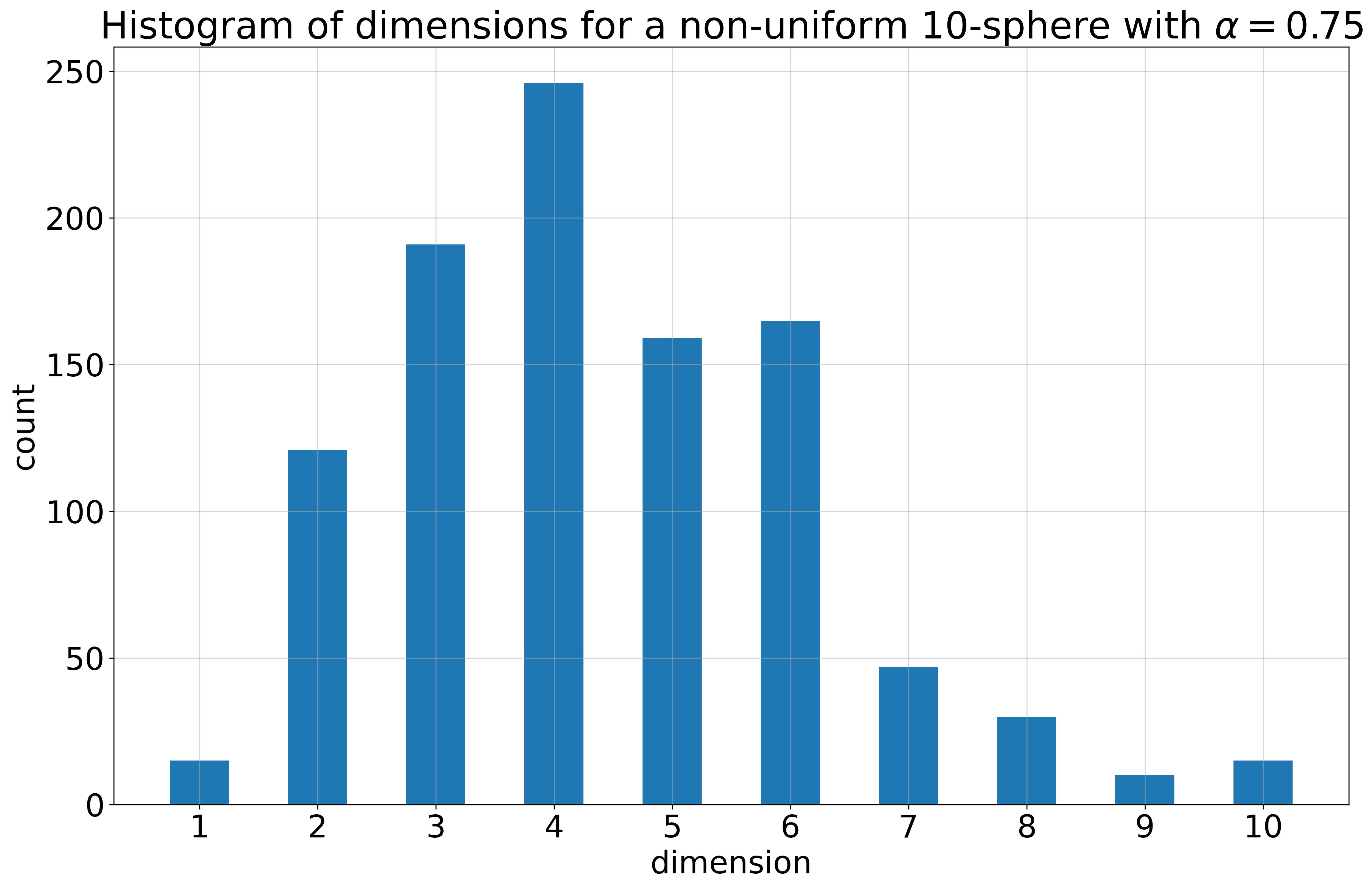}
         \includegraphics[width=\linewidth]{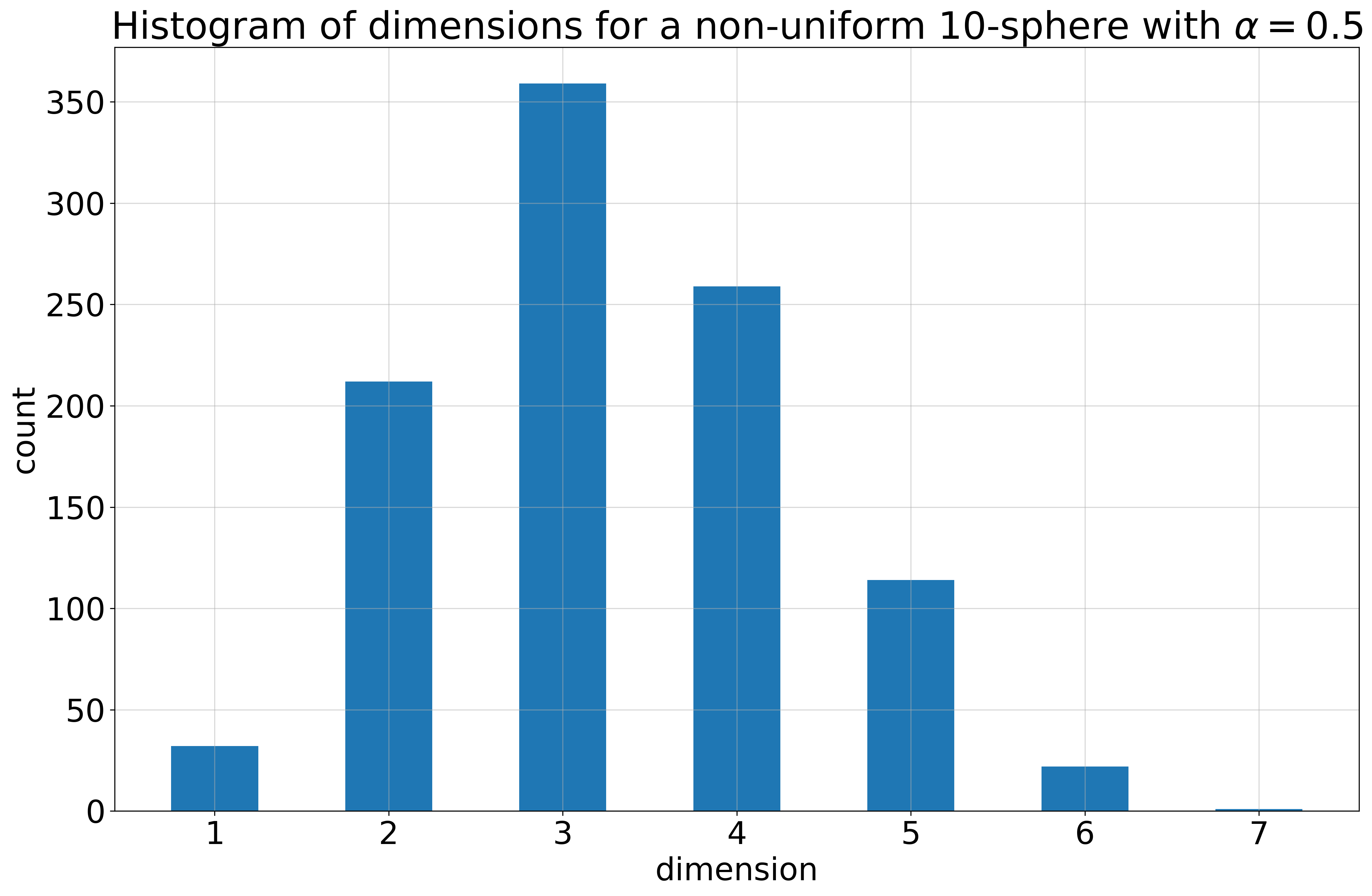}
     \end{minipage}
     \caption{Dimensionality estimates for uniform and non-uniform distributions on a 10-sphere. On the right we present histograms showing how many points $\textbf{x}_0^{(j)}$ result in a given $\hat{k}(\textbf{x}_0^{(j)})$. Taking $\hat{k} = \max_j \hat{k}(\textbf{x}_0^{(j)})$ allows for robust estimation for moderate values of the concentration parameter $\alpha$.}
     \label{fig:non_uniform}
\end{figure*} 

\subsection{Relaxing the strict manifold assumption}
The proof of the Theorem \ref{thm:score_orthogonal} assumes that $p_0$ is strictly contained on the data manifold $\mathcal{M}$, however in practice it is possible that the data distribution is concentrated around $\mathcal{M}$ rather then being contained within. Therefore, we conduct an empirical analysis, which examines how our method works in the case of the data contained around the manifold. We start with $p_0$ a uniform distribution on a unit $25$-sphere embedded in a 100 dimensional space and convolve it with a Gaussian distribution $\mathcal{N}(0, \sigma \mathbf{I})$ to obtain a distribution $p_0^\sigma$ that concentrates around $\mathcal{M}$. As $\sigma$ increases the distribution is blurred out more and more into the ambient space. We train score model on each of the resulting distributions and use our method to estimate the dimension. We find that our method produces correct estimation for for small values of $\sigma$ i.e. when $p_0^\sigma$ is still concentrated tightly around $\mathcal{M}$. This is expected since of high values of $\sigma$ the distribution isn't really concentrated around any manifold and therefore the notion of intrinsic dimension doesn't make any sense. The results are presented in Figure \ref{fig:robustness_samples}.

\begin{figure}
    \centering
    \includegraphics[width=\textwidth]{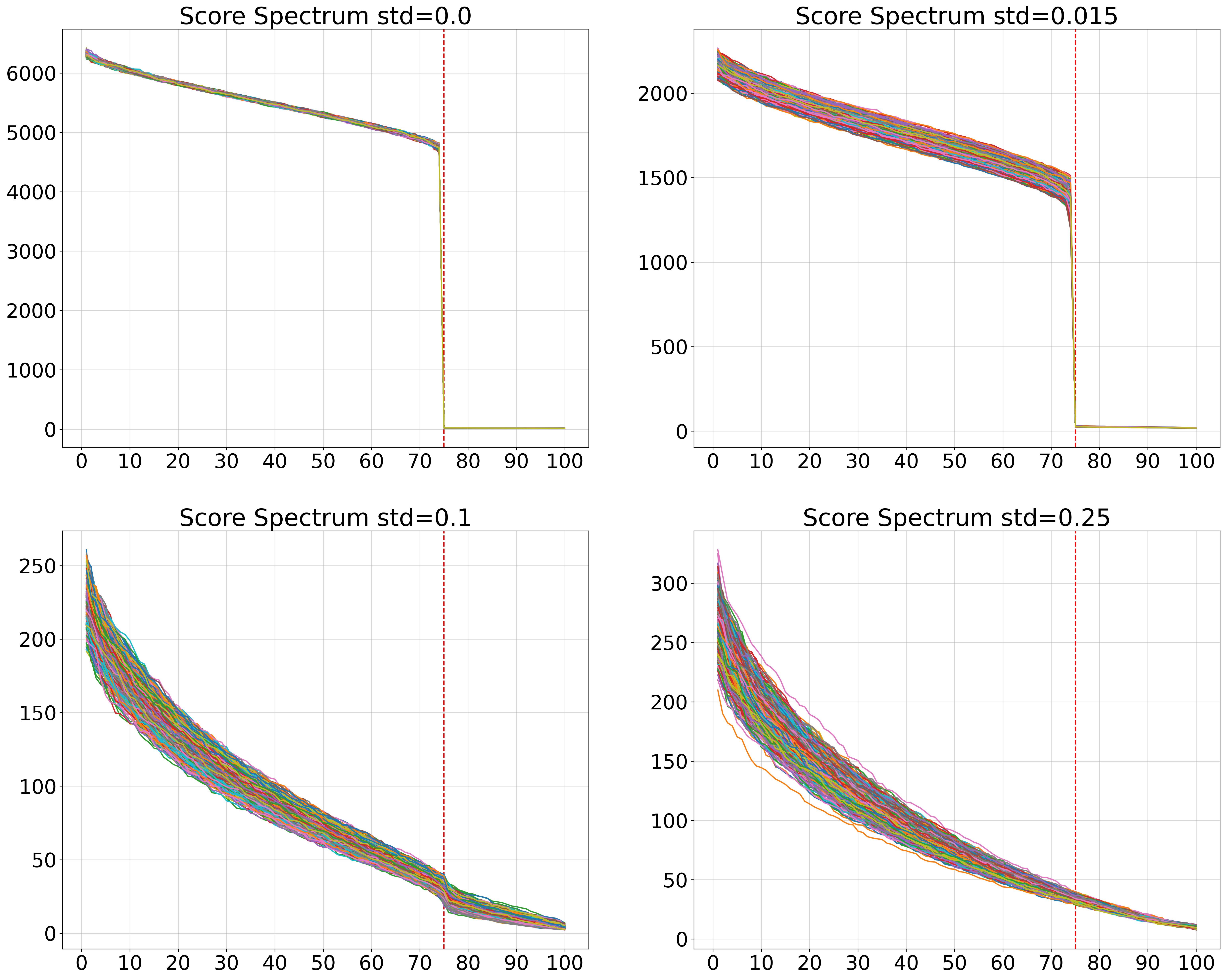}
    \caption{Score spectra for score models on 25-sphere trained on noisy manifold data.}
    \label{fig:robustness_samples}
\end{figure}


\end{document}